\documentclass[12pt]{article} 
\usepackage{amsmath,amsthm,amssymb,bm,mathrsfs,}
\usepackage{algorithm,algorithmic,graphicx}
\usepackage{caption,enumerate,natbib,listings,setspace}
\usepackage[OT1]{fontenc}
\usepackage[colorlinks,linkcolor=blue,anchorcolor=blue,citecolor=blue,urlcolor=blue]{hyperref}
\usepackage[margin=1in]{geometry}
\usepackage[protrusion=false,expansion=true]{microtype}
\usepackage[title]{appendix}
\usepackage{bbm}
\usepackage{comment}
\theoremstyle{plain}
\let\hat\widehat
\let\tilde\widetilde

\newcommand{\argmin}{\mathop{\mathrm{argmin}}}
\newcommand{\argmax}{\mathop{\mathrm{argmax}}}

\newtheorem{lemma}{{\bf Lemma}}

\newtheorem{corollary}{{\bf Corollary}}

\newtheorem{theorem}{{\bf Theorem}}

\newtheorem{assumption}{{\bf Assumption}}

\newtheorem{definition}{{\bf Definition}}

\newtheorem{remark}{{\bf Remark}}

\usepackage{xr}
\usepackage{color}
\usepackage{subcaption}
\usepackage{float}

\title{\Large{\textbf{Privacy-Preserving Dynamic Assortment Selection}} }

\author{Young Hyun Cho\thanks{Department of Statistics, Purdue University, Email: cho472@purdue.edu} \text{ and }Will Wei Sun\thanks{Mitchell E. Daniels, Jr. School of Business, Purdue University. Email: sun244@purdue.edu. Corresponding author.} }

\date{}
\begin{document} 

\maketitle

\begin{abstract}
\noindent

\end{abstract}
With the growing demand for personalized assortment recommendations, concerns over data privacy have intensified, highlighting the urgent need for effective privacy-preserving strategies. This paper presents a novel framework for privacy-preserving dynamic assortment selection using the multinomial logit (MNL) bandits model. Our approach employs a perturbed upper confidence bound method, integrating calibrated noise into user utility estimates to balance between exploration and exploitation while ensuring robust privacy protection. We rigorously prove that our policy satisfies Joint Differential Privacy (JDP), which better suits dynamic environments than traditional differential privacy, effectively mitigating inference attack risks. This analysis is built upon a novel objective perturbation technique tailored for MNL bandits, which is also of independent interest. Theoretically, we derive a near-optimal regret bound of \(\tilde{O}(\sqrt{T})\) for our policy and explicitly quantify how privacy protection impacts regret. Through extensive simulations and an application to the Expedia hotel dataset, we demonstrate substantial performance enhancements over the benchmark method.



\newpage
\baselineskip=25pt 

\section{Introduction}
\label{sec:introduction}

Assortment optimization has become a crucial element of online retail platforms, significantly influencing product recommendation on e-commerce websites, marketing campaigns, and promotional slots. The primary objective is to select a subset of items that maximizes expected revenue based on a user’s choice model. With the advent of big data and the widespread adoption of machine learning algorithms, personalized recommendations have become feasible, enabling companies to present relevant products to the right audience and significantly enhancing potential revenue.  In a recent McKinsey report\footnote{\url{https://www.mckinsey.com/capabilities/growth-marketing-and-sales/our-insights/the-value-of-getting-personalization-right-or-wrong-is-multiplying}}, consumer demand for personalization is growing, with 71\% of consumers expecting personalized interactions and 76\% expressing frustration when these expectations are not met.

The literature on revenue management increasingly addresses the development of assortment optimization algorithms in online retail, commonly referred to as dynamic assortment selection. This involves dynamically optimizing and adapting product offerings in response to changing conditions, leveraging customer data such as past purchases, browsing history, demographics, and other relevant information. However, the rapid growth of personalized recommendations has also led to heightened concerns about privacy \citep{bi2023distribution,su2024statistical}. A KPMG survey\footnote{\url{https://kpmg.com/us/en/articles/2023/bridging-the-trust-chasm.html}} revealed that 70\% of companies have increased their collection of personal consumer data, while 86\% of individuals expressed growing concerns about data privacy. Additionally, 40\% of individuals do not trust companies to use their data ethically, and 30\% are unwilling to share personal data for any reason. These concerns have prompted stringent regulations such as the General Data Protection Regulation (GDPR)\footnote{\url{https://www.consilium.europa.eu/en/policies/data-protection/data-protection-regulation/}} in the European Union and the California Consumer Privacy Act (CCPA)\footnote{\url{https://oag.ca.gov/privacy/ccpa}}. Moreover, in 2022, China introduced the ``Internet Information Service Algorithmic Recommendation Management Provisions"\footnote{\url{https://www.chinalawtranslate.com/en/algorithms/}} to regulate algorithmic recommendations. This policy emphasizes the need for socially responsible model design and mandates regular outcome reviews to safeguard personal information.


\begin{figure}[htbp]
\centering
\includegraphics[scale=0.6]{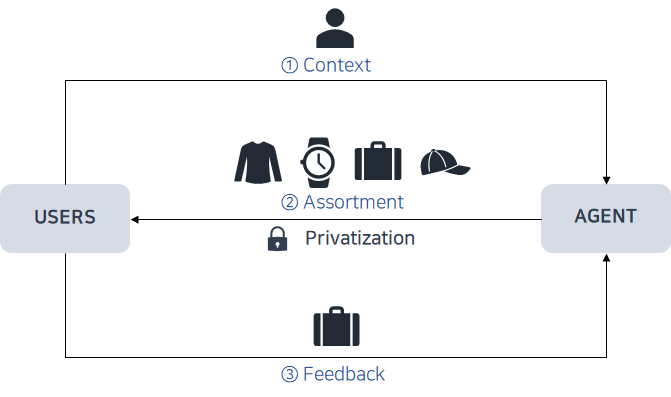}
\caption{Flow of Privacy-Preserving Dynamic Assortment Selection}
\label{fig: bandit loop}
\end{figure} 

This paper explores privacy-preserving personalized dynamic assortment selection, a crucial issue in online retail platforms. Dynamic assortment selection has been effectively modeled using the multinomial logit (MNL) bandits framework, where the customer's choice model is represented by a multinomial logit model \citep{chen2020dynamic, oh2019thompson, oh2021multinomial,dong2023pasta,lee2024nearly,zhang2024online}. In the MNL bandit problem, as depicted in Figure \ref{fig: bandit loop}, a potential customer arrives at each time step \( t \), and the agent observes the user's contextual information. Based on this information and historical feedback data, the agent offers a personalized assortment of items. The user then decides whether to purchase one of the items or to make no purchase at all. The purchase decision serves as feedback to the agent, enabling it to update its policy and optimize the assortment for future users. 

While online platforms recognize the importance of privacy and typically employ secure communication channels to safeguard user data from third-party access, the risk of privacy breaches remains, particularly through ``inference attacks" \citep{chen2022privacy,lei2023privacy,xu2024rate}. Even when customer data is not shared with third parties, malicious adversaries can either infiltrate the system by pretending to be users or collude with users to infer sensitive information. These adversaries can share their data including prescribed outputs and leverage observed fluctuations in these outputs to infer information about other users. For instance, if an item that was not frequently recommended before time \( t \) suddenly starts being recommended, it can be inferred that the user at time \( t \) may have purchased that item. Furthermore, if that item is commonly associated with customers who share certain features, it may also be possible to infer the specific characteristics of the user at time \( t \).

The primary challenge in privacy-preserving online decision-making lies in managing the trade-off between exploration and exploitation while preventing adversaries from inferring sensitive information. At each time step, the agent must decide whether to exploit existing knowledge gained from previous observations or to explore lesser-known options that may yield better outcomes. Since the agent only receives feedback on the items it recommends, this feedback is inherently partial, leaving the agent with incomplete knowledge of the overall item set even after multiple rounds, which can lead to suboptimal recommendations without adequate exploration. Privacy concerns further complicate this problem. Assortments based on sensitive user information, if exposed, could assist adversaries' inference attacks. To mitigate this risk, additional randomness needs to be introduced in optimizing the assortments to obscure the sensitive information. However, it is critical to decide the correct level of randomness, as too much randomness may also lead to suboptimal personalization decisions.


To address these intertwined challenges, we introduce a perturbed upper confidence bound (UCB) method. Although online learning methods are advancing rapidly \citep{shi2023value, shi2023dynamic, qi2024proximal, zhou2024policy}, the UCB approach \citep{lai1985asymptotically} remains a core technique in bandit and reinforcement learning, effectively managing the exploration-exploitation trade-off by giving higher exploration bonuses to underexplored actions. In the context of privacy preservation, we extend UCB by adding calibrated noise to the user utility estimates, resulting in a perturbed UCB framework. This additional layer of randomness increases the complexity of the problem, requiring the design of an algorithm that is able to explicitly quantify the trade-off between the randomness introduced to preserve privacy and the degradation in personalization.


To formally quantify privacy protection, we ensure that our policy satisfies Joint Differential Privacy (JDP) \citep{hsu2014private}, which is a relaxed standard derived from differential privacy (DP) \citep{dwork2006differential}, designed to be more suitable for online decision-making settings, such as bandit algorithms \citep{shariff2018differentially, chen2022privacy} and reinforcement learning \citep{vietri2020private, zhou2022differentially}. DP ensures individual privacy by requiring that mechanisms applied to neighboring datasets differing by a single entity yield similar distributions. However, in the context of dynamic assortment selection, where the output of bandit policy is a sequence of assortments for all users, DP enforces similar recommendations across all the users, which conflicts with the goal of personalized decisions. Moreover, consideration of the entire sequence of assortments in DP is unrealistic in practical online platforms. The idea behind targeting the entire sequence is to ensure that the adversary, who can observe the full sequence of recommendations, is unable to perform an inference attack. However, in real-world online platforms where secure channels are typically used, the adversary does not have access to a user’s assortment unless they are colluding with that user. This makes DP overly stringent for our setting, as it assumes a level of exposure that rarely occurs in practice. JDP, on the other hand, requires that replacing any single user does not significantly affect the assortments for the remaining \(T-1\) users. Thus, it safeguards against inference attacks by adversaries colluding with the \(T-1\) users, excluding the target user. Since the target user's assortment is not taken into account, JDP not only better reflects real-world scenarios but also allows for personalized recommendations based on the target user's sensitive information.

\subsection{Our Contributions}

Our contributions can be categorized into two main areas: methodological advancements and theoretical developments.

\noindent
\textbf{Methodological Contributions:} We propose the first privacy-preserving dynamic assortment selection policy via the MNL contextual bandit framework that satisfies JDP. We highlight that addressing the MNL model requires novel solutions, as existing privacy-preserving methods for bandit problems based on linear and generalized linear models (GLM) are not applicable to this setting. For instance, privacy-preserving linear contextual bandit studied in \cite{shariff2018differentially} extensively relies on a tree-based aggregation \citep{chan2011private} which continuously releases private sum statistics, leveraging that its UCB has sufficient statistics in the form of sum statistics. However, this approach is infeasible for our MNL bandits due to the lack of such explicit sufficient statistics. On the other hand, while GLM bandit studied in \cite{chen2022privacy} relies on private MLE through objective perturbation \citep{chaudhuri2011differentially,kifer2012private}, there is no established technique in objective perturbation that is applicable to our multinomial models. To bridge this gap, we develop a novel objective perturbation technique tailored to MNL bandits, which is also of independent interest in the field of privacy.

We reformulate the presentation of Joint Differential Privacy (JDP) to minimize the noise required for privacy protection in our MNL bandit model. Unlike traditional DP, which has various formulations \citep{bun2016concentrated, dong2022gaussian, su2024statistical}, JDP has traditionally relied on a likelihood ratio-based definition that is less efficient in online settings. In online learning, it is often necessary to compose multiple privacy mechanisms to safeguard newly arising sensitive information. This challenge is particularly pronounced in our multinomial models, where privacy concerns involve multiple items simultaneously, in contrast to linear contextual bandits or GLM bandits, which focus on a single item at a time. To tackle this issue, we employ R\'enyi divergence \citep{renyi1961measures} to establish tighter composition bounds, thereby achieving a more efficient privacy mechanism. Empirical results demonstrate the effectiveness of this refined approach.


\noindent
\textbf{Theoretical Contributions:} 
In theory, we establish the regret bound of our algorithm. Regret measures the cumulative difference between the rewards achieved by the algorithm and those achieved by an optimal policy with full knowledge of the environment. In our case, it reflects the gap between the expected reward of the optimal assortment and that of our privacy-preserving policy. Our policy achieves a regret bound of \(\Tilde{O}\left(\left(d+\frac{d^{5/2}}{\rho_{1}}+\frac{d^{3/4}}{\rho_{2}^{1/4}}\right)\sqrt{T}\right),\) where $d$ represents the dimension of contextual vectors, $T$ denotes the entire time horizon, and \(\rho_{1}\) and \(\rho_{2}\) are the privacy parameters. Here \(\rho_{1}\) is associated with the \texttt{PrivateMLE} subroutine, and \(\rho_{2}\) corresponds to the \texttt{PrivateCov} subroutine. This bound matches the \(\Omega(\sqrt{T})\) lower bond of MNL contextual bandit policy established in \citep{chen2018note}, with other terms held constant. Importantly, this regret bound also suggests that \(\rho_{1}\) has a greater influence on regret than \(\rho_{2}\), suggesting that, for a fixed privacy budget, a larger portion should be allocated to \(\rho_{1}\). This practical guideline for privacy budget allocation is further validated through extensive experiments.


Our proof technique diverges significantly from previous work, such as \cite{chen2022privacy}, which studied privacy in GLM bandits but did not offer a detailed analysis of how each subroutine's privacy budget impacts regret. In contrast, we provide a more nuanced analysis by explicitly quantifying the contribution of each subroutine to the overall regret. Moreover, our privacy analysis builds upon the newly developed objective perturbation technique tailored to our MNL bandits. It leverages advanced optimization methods and random matrix theory to effectively handle the noise introduced in each subroutine to guarantee privacy.


\subsection{Related Literature}
Our work is closely related to dynamic assortment selection and differentially private contextual bandits literature. In the following, we review these two areas and highlight how our approach differs from existing studies.

\noindent
\textbf{Dynamic Assortment Selection:} The dynamic assortment problem with unknown customer preferences began with \cite{caro2007dynamic}. While noncontextual MNL models have been extensively studied \citep{rusmevichientong2010dynamic, saure2013optimal, agrawal2017thompson, agrawal2019mnl, chen2018note}, recent works focus on personalized contexts \citep{article, chen2020dynamic, ou2018multinomial, oh2019thompson, oh2021multinomial,lee2024low,lee2024nearly}. UCB-based policies for MNL contextual bandits were first introduced by \cite{chen2020dynamic}, where UCB was applied to the expected revenue of each possible assortment. 
In a different approach, \cite{oh2021multinomial} applied UCB to each item's utility individually, selecting the assortment that maximizes the expected revenue based on these optimistic utility estimates. 
Despite inherent privacy concerns, privacy in dynamic assortment selection remains unexplored. We address this gap by building on \cite{oh2021multinomial} to develop the first private dynamic assortment selection algorithm. Unlike \cite{lei2023privacy} which deals with offline personalized assortment, our method handles online and adaptive assortments.

\noindent
\textbf{Differentially Private Contextual Bandits:} The literature on private contextual bandits is rapidly emerging, with a focus on both linear and generalized linear bandits. There are two primary privacy regimes in this context. The first is the Central DP regime, where a trusted agent manages raw user data and injects noise to ensure privacy. On the other hand, the Local DP regime \citep{zheng2020locally, han2021generalized, chen2022privacy} assumes no trusted agent; users themselves perturb the signal before communicating with the central server, preventing the server from accessing raw data. Our paper considers the central DP. In the Central DP, perturbed UCB-based policies have been studied for both linear contextual bandit \citep{shariff2018differentially} and generalized linear bandit \citep{chen2022privacy, su2023differentially}. Our work is closely related to \cite{chen2022privacy}, which studies GLM bandits in the Central DP regime using UCB-based strategy. However, our multinomial model introduces higher technical complexity in both the bandit formulation and privacy preservation, as the objective perturbation technique used in \cite{chen2022privacy} is not applicable to the multinomial setting. Another key distinction from \cite{shariff2018differentially} and \cite{chen2022privacy} is their adoption of Anticipated Differential Privacy (ADP), which is a weaker guarantee than our JDP. 
While ADP ensures that adversaries cannot infer sensitive information about a target user from future users, it does not protect against collusion with earlier users. In contrast, JDP provides a stronger guarantee by protecting against both past and future collusions.

\subsection{Notation}
Throughout this paper, we denote $[T] = \{1,2,\cdots,T\}$ for any positive integer $T$. For a vector $x \in \mathbb{R}^{d},$ $\Vert x \Vert$ denotes its $l_{2}$-norm. The weighted $l_{2}$-norm with respect to a positive-definite matrix $V$ is defined by $\Vert x \Vert_{V} = \sqrt{x^\top V x}.$ The minimum and maximum eigenvalues of a symmetric matrix $V$ are written as $\lambda_{\text{min}}(V)$ and $\lambda_{\text{max}}(V)$, respectively. For two positive sequences $\{a_{n}\}_{n\geq1}$ and $\{b_{n}\}_{n\geq1}$, we say $a_{n} = \mathcal{O}(b_{n})$ if $a_{n} \leq C b_{n}$ for some positive constant $C$ for all large $n$. We let $\Tilde{\mathcal{O}}(\cdot)$ represent the same meaning of $\mathcal{O}(\cdot)$ except for ignoring a log factor. In addition, $a_{n} = \Omega(b_{n})$ if $a_{n} \geq C b_{n}$ for some positive constant $C$ for all large $n$.


\section{Problem Setting}
\label{sec: problem setting}

In this section, we first introduce the problem of dynamic assortment selection in a MNL contextual bandit setting, describing the modeling of user preferences and the objective of minimizing the expected cumulative regret. We then formalize the privacy requirements by defining DP and JDP using the R\'enyi divergence, discuss their application to JDP, and present key theoretical results that facilitate the design of a privacy-preserving policy.

\subsection{Multinomial Logit Contextual Bandit}

At each round $t \in [T]$, a customer comes to the platform and the agent observes feature vector $x_{ti} \in \mathbb{R}^{d}$ for all items $i \in [N]$ which contains contextual information on both customer and item. Using the given information as well as the historical interactions until $t-1$ period, the agent offers an assortment set $S_{t} = \{i_1, \dots, i_K\} \in \mathcal{S}$ that consists of $K$ items out of $N$ total items among the set of candidate assortments $\mathcal{S}$. Then the agent observes the user's purchase decision $c_{t} \in S_{t} \cup \{0\},$ where $\{0\}$ denotes the case the user does not purchase any item offered in $S_{t}$. This decision serves as feedback for the agent to update its knowledge on users' demand on items. 

To model users' preference, we consider a widely adopted multinomial logit (MNL) choice model \citep{mcfadden1972conditional}, where the choice probability for item $i_{k} \in S_{t} \cup \{0\}$ is defined as 
\begin{equation*}\label{equ:choice_probability}
    p(i_{k} \mid S_{t}, \theta^{\ast}) = \frac{\exp\left\{x_{ti_k}^\top \theta^{\ast}\right\}}{1 + \sum_{j \in S_{t}} \exp\left\{x_{tj}^\top \theta^{\ast}\right\}}, \quad
    p(0 \mid S_{t}, \theta^{\ast}) = \frac{1}{1 + \sum_{j \in S_{t}} \exp\left\{x_{tj}^\top \theta^{\ast}\right\}},
\end{equation*}
where $\theta^{\ast}$ is a true time-invariant parameter unknown to the agent. Note that $x_{ti}^\top \theta^{\ast}$ can be understood as $t^{th}$ user's \emph{utility parameter} for product $i$. The MNL model assumes that choice response for items in $S_{t}$ is sampled from the following multinomial distribution:
\begin{equation*}
 y_{t} = (y_{t0},y_{t1},\dots,y_{ti_K}) \sim \text{Multi}\{1,p_{t}(0 \vert S_t,\theta^{\ast}),p_{t}(i_1 \vert S_t,\theta^{\ast}),\dots,p_{t}(i_K \vert S_t,\theta^{\ast})\}.   
\end{equation*}
Moreover, we denote the noise $\epsilon_{ti} = y_{ti} - p_{t}(i \vert S_{t},\theta^{\ast})$ for $i \in S_{t} \cup \{0\} $ and $t \in [T]$. As $\epsilon_{ti}$ is a bounded random variable, it is $\gamma^{2}$-subgaussian random variable with $\gamma^{2} = 1/4$.

Following \cite{oh2021multinomial}, we assume the revenue $r_{ti}$'s are public information to the agent and satisfy $\vert r_{t,i} \vert \leq 1$. The expected revenue of the assortment $S_{t}$ is given by
  \begin{equation*}
   R_{t}(S_{t},\theta^{\ast}) = \sum_{i \in S_{t}} r_{ti}p_{t}(i \vert S_{t},\theta^{\ast}).
  \end{equation*}

For a dynamic assortment selection policy that sequentially decides the assortment $S_{t} \in \mathcal{S}$, its performance is measured by \emph{expected cumulative regret} which is the gap between the expected revenue of this policy and that of the optimal assortments in hindsight: 
\begin{equation*}
    \mathcal{R}_{T} = \mathbb{E} \left[ \sum_{t=1}^{T} \left( R_{t}(S_{t}^{\ast},\theta^{\ast}) - R_{t}(S_{t},\theta^{\ast}) \right)
\right],
\end{equation*}
where $S_{t}^{\ast} = \argmax_{S \subset \mathcal{S}}R_{t}(S,\theta^{\ast})$ is the oracle optimal assortment at time $t$ when $\theta^{\ast}$ is known apriori. The goal of dynamic assortment selection is to minimize the expected cumulative over the total time horizon $T$. 
\subsection{Joint Differential Privacy}
Unlike the existing literature in MNL contextual bandit \citep{chen2020dynamic,oh2019thompson,oh2021multinomial,lee2024low,lee2024nearly}, we aim to design a private dynamic assortment selection policy that satisfies JDP which, while being a relaxed form of DP, still provides strong enough protection against inference attack. Both JDP and DP require that outputs from two neighboring datasets, differing by only one entry, have similar distributions. In bandit problems with a finite time horizon \(T\), a sequence of \(T\) users constitutes the dataset \(U = \{u_{t}\}_{t=1}^{T}\), where each user \(u_{t} \in \mathcal{U}\) is characterized by their context, the prescribed assortment, and their purchase decision, \(\{\{x_{ti}\}_{i \in [N]}, S_{t}, y_{t}\}\). To provide privacy for any  $t^{th}$ user, we consider neighboring datasets as those that differ by a single user at $t$ while sharing the remaining \(T-1\) users.

\begin{definition}\label{def: neighboring data}($t$-neighboring datasets). Two datasets \(U \in \mathcal{U}^{T}\) and \(U' \in \mathcal{U}^{T}\) are said to be $t$-neighboring if they differ only in their \(t\)-th entry.
\end{definition}

\begin{figure}[htbp]
    \centering
    \includegraphics[width=\textwidth]{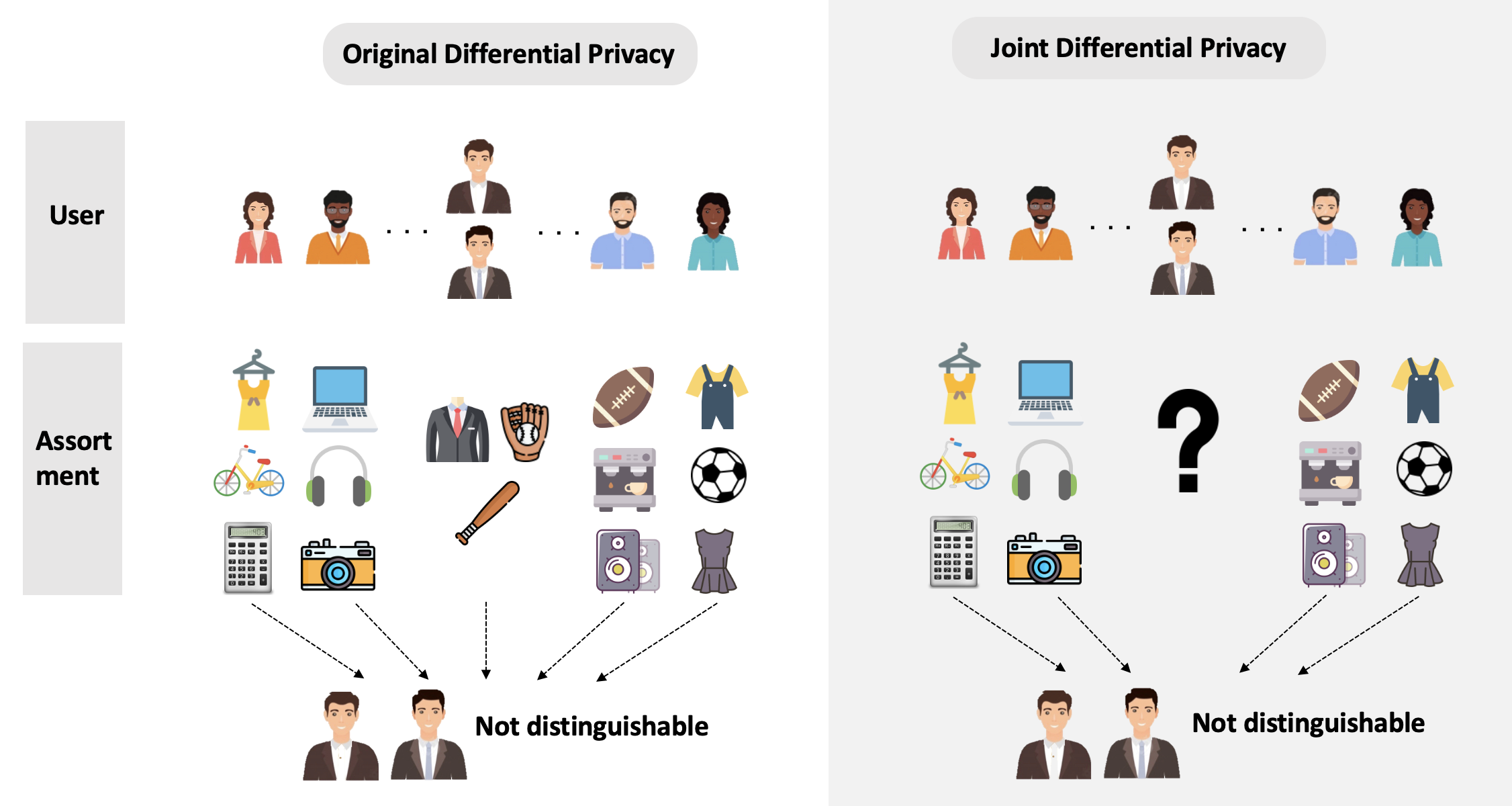}
    \caption{Comparison of DP and JDP: DP ensures that the entire sequence of assortments remains similar across neighboring datasets, while JDP requires similarity for assortments excluding the target user's, allowing for more flexible personalized recommendation.}
    \label{fig: jdp def}
\end{figure}

\begin{definition}\label{def: zcdp_jdp_combined} 
A contextual bandit algorithm \(\mathcal{M}: \mathcal{U}^{T} \rightarrow \mathcal{S}^{T}\) satisfies \(\rho\)-zero concentrated differential privacy (\(\rho\)-zCDP) \citep{bun2016concentrated} if for any $t$-neighboring user sequences \(U\) and \(U'\),
    \[
    D_{\alpha}\left(\mathcal{M}(U) \Vert \mathcal{M}(U')\right) \leq \rho \alpha, \text{ for all } \alpha > 1,
    \]
    where $D_{\alpha}(P \Vert Q) = \frac{1}{\alpha -1}\log \left(\sum_{E}P(E)^{\alpha}Q(E)^{(1-\alpha)}\right)$ is the R\'enyi divergence of order $\alpha$ of the distribution $P$ from the distribution $Q$.
    
    On the other hand, \(\mathcal{M}\) satisfies \(\rho\)-joint zCDP if for any $t$-neighboring user sequences \(U\) and \(U'\),
    \[
    D_{\alpha}\left(\mathcal{M}_{-t}(U) \Vert \mathcal{M}_{-t}(U')\right) \leq \rho \alpha, \text{ for all } \alpha > 1,
    \]
    where \(\mathcal{M}_{-t}(U)\) represents all outputs except for the assortment recommended to user \(t\).

\end{definition}

In Definition \ref{def: zcdp_jdp_combined}, for both $\rho$-zCDP and $\rho$-joint zCDP, the distributions are constrained to be similar with respect to R\'enyi divergence with the privacy parameter \(\rho\) that governs the level of similarity, with smaller values of \(\rho\) indicating greater similarity between the distributions and, thus, stronger privacy protection. As illustrated in Figure \ref{fig: jdp def}, the fundamental difference between DP and JDP lies in how they regulate the similarity of outputs from neighboring datasets. DP ensures that the entire sequence of assortments produced by the bandit policy across \(T\) time steps is similar.  In contrast, JDP requires that the sequence of assortments for the \(T-1\) users, excluding the prescribed assortment for user \(t\), exhibits similar distributions.

Note that the consideration of the entire sequence of assortments in $\rho$-zCDP inevitably results in similar assortments across all users, which contradicts the personalized nature of contextual bandits. Indeed, \cite{shariff2018differentially} and \cite{chen2022privacy} showed that linear contextual bandits and GLM bandits that satisfying DP have lower bound of $\Omega(T)$ for their regrets, which is undesirable. JDP, on the other hand, offers flexibility for personalized recommendations by not requiring the consideration of \(S_{t}\). This is illustrated in Lemma \ref{thm: billboard lemma}, which demonstrates that if a bandit mechanism’s action at each time step is governed by subroutines satisfying DP and the true data of $t$ user, the entire mechanism satisfies JDP.

\begin{lemma}\label{thm: billboard lemma}(Billboard Lemma \citep{hsu2014private}) 
    Suppose \( M: \mathcal{U}^{T} \rightarrow \mathcal{R} \) satisfies \(\rho\)-zCDP. Consider any set of deterministic functions \( f_{t}:\mathcal{U}_{t} \times \mathcal{R} \rightarrow \mathcal{R}' \), where \(\mathcal{U}_{t}\) represents the data associated with the \(t^{th}\) user. The composition \(\mathcal{M}(U) = \{f_{t}(U_{t}, M(U))\}_{t=1}^{T}\) satisfies \(\rho\)-joint zCDP, where \(U_{t}\) denotes the subset of the full dataset \(U\) corresponding to the \(t^{th}\) user’s data.
\end{lemma}

In particular, Lemma \ref{thm: billboard lemma} enables our algorithm to incorporate the true context vector of the user $t$, which is crucial for making personalized recommendations. While the overall policy ensures JDP, each privacy subroutine is carefully designed to satisfy DP. In summary, our goal is to construct a dynamic assortment selection policy that not only satisfies \(\rho\)-joint zCDP but also minimizes expected cumulative regret, thereby ensuring both privacy protection and personalized assortment recommendation.

Before concluding this section, we provide two key observations on the privacy definition. First, to the best of our knowledge, this is the first work to employ \(\rho\)-zCDP for JDP. Previous research on privacy-preserving bandit \citep{shariff2018differentially, chen2022privacy} and reinforcement learning problems \citep{vietri2020private, zhou2022differentially} typically relies on standard \((\epsilon,\delta)\)-DP language which bounds the likelihood ratio of outputs between neighboring datasets. While theoretically valid, this approach is not efficient in practice. In online settings, where multiple compositions of privacy mechanisms are required, the privacy parameters \(\epsilon\) and \(\delta\) in \((\epsilon,\delta)\)-DP tend to inflate conservatively with each composition. In contrast, \(\rho\)-zCDP, based on R\'enyi divergence, offers tighter composition bounds and less inflation of privacy parameter $\rho$, as shown in \cite{bun2016concentrated}. This results in a more efficient privacy mechanism with reduced noise, while maintaining equivalent levels of privacy protection. Additionally, \(\rho\)-zCDP provides stronger privacy guarantees, as it can be converted into \((\epsilon,\delta)\)-DP, but not the reverse. Therefore, our use of \(\rho\)-zCDP for JDP results in satisfying privacy protection with less noise than existing approaches. 

Lastly, we underscore the importance of selecting the appropriate adjacency in JDP. In privacy mechanisms, adjacency refers to how two datasets are considered neighboring. There are two primary types: bounded and unbounded type. Bounded type assumes that neighboring datasets differ by exactly one entity while  having the same number of entries, which is suitable when the total number of entries is known or not sensitive. Unbounded type allows neighboring datasets to differ by the presence or absence of an entity, which is relevant when the number of entities itself is sensitive. It is crucial to select the correct notion of adjacency, as the two are not interchangeable, and using the wrong one can lead to flawed privacy analysis. Our privacy guarantee is based on bounded type adjacency, which is suited for scenarios like inference attacks, where the focus is on protecting changes in user interactions rather than their mere presence.


\section{Methodology}
\label{sec: methodology}
In this section, we present a privacy-preserving MNL bandit policy \texttt{DPMNL} (Algorithm \ref{alg: dpmnl}), supported by two essential subroutines: \texttt{PrivateMLE} (Algorithm \ref{alg: private mle}) and \texttt{PrivateCov} (Algorithm \ref{alg: privatecov}). 
\begin{figure}[h!]
    \centering
    \includegraphics[scale=0.6]{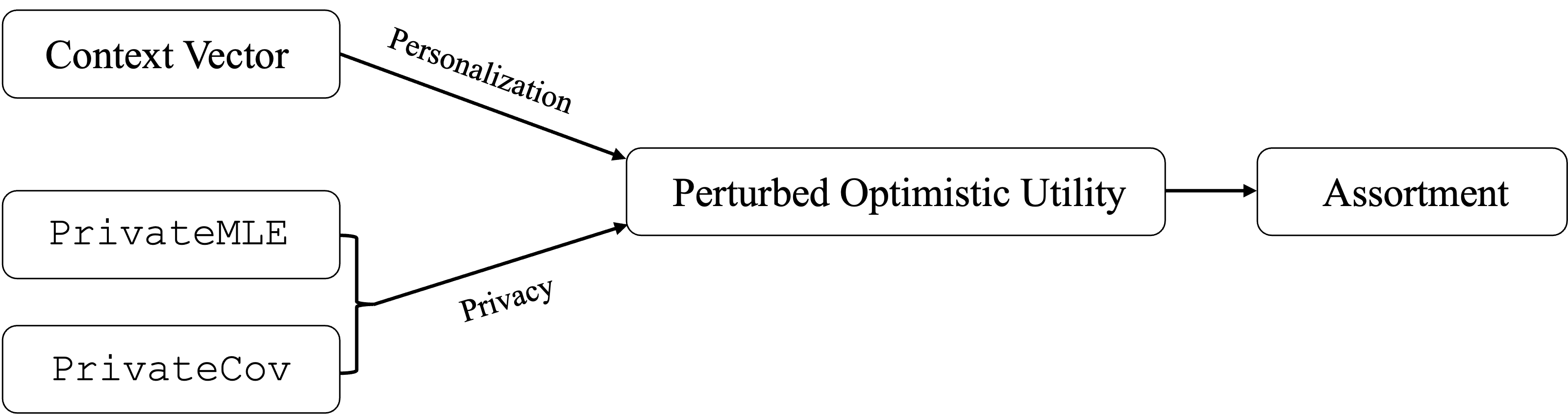}
    \caption{Overview of the mechanism design.}
    \label{fig: bigpicture}
\end{figure}
Figure \ref{fig: bigpicture} illustrates the overall mechanism design. 
\texttt{PrivateMLE} satisfies \(\rho_{1}\)-zCDP by constructing a private MLE for the model parameter \(\theta^{\ast}\), allowing \( x_{ti}^\top \hat{\theta} \) to serve as the point estimate of the utility of the \(i^{th}\) item for user \(t\). \texttt{PrivateCov}, on the other hand, encapsulates both the noisy matrix required for $\rho_{2}$-zCDP guarantees and the contextual information from historical users' assortments. This subroutine primarily contributes to the computation of exploration bonuses. Moreover, to provide personalized recommendations, the true context vector of the user at time \(t\) is utilized without perturbation. By combining these components, the policy computes perturbed optimistic utilities for each item, and selects the assortment with the highest estimated revenue.

\subsection{Main Algorithm: \texttt{DPMNL}}

We present the detailed explanation on \texttt{DPMNL} algorithm, which balances exploration and exploitation using a perturbed optimistic utility approach while ensuring privacy. The algorithm operates in two main phases: a pure exploration phase followed by an exploration-exploitation phase driven by the perturbed UCB strategy.

\begin{algorithm}[htb!]
\caption{Main Algorithm: \texttt{DPMNL}}\label{alg: dpmnl}
\begin{algorithmic}[1]
    \STATE \textbf{Input: } privacy parameters $\rho_{1}$, $\rho_{2}$, number of pure-exploration periods $T_{0}$, maximum number of \texttt{PrivateMLE} $D_{\text{MLE}}$, regularization parameter $\lambda$, confidence radius $a_{t}$
    \STATE \textbf{Pure exploration phases}
    \FOR{$t \in [T_{0}]$}
    \STATE A user comes with context $\{x_{ti}\}_{i \in N}$
    \STATE Randomly choose $S_{t}$ with $\vert S_{t} \vert = K$ and observe user feedback vector $y_{t}$
    \STATE $V_{t} = V_{t-1} + \sum_{i \in S_{t}}x_{ti}x_{ti}^\top$
    \ENDFOR
    \STATE Compute MLE $\hat{\theta}_{T_{0}} = \texttt{PrivateMLE}(T_{0},\frac{\rho_{1}}{D_{\text{MLE}}})$ via Algorithm \ref{alg: private mle} and \\
    Compute $V_{T_{0}} = \texttt{PrivateCov}(T_{0},\rho_{2}) + 2\lambda I$ via Algorithm \ref{alg: privatecov}
    \STATE \textbf{Exploration and exploitation phases}
    \FOR{$t = T_{0}+1$ to $T$} 
    \STATE A user comes with context $\{x_{ti}\}_{i \in N}$
    \STATE Compute $z_{ti} = x_{ti}^\top \hat{\theta}_{t-1} + \alpha_{t}\Vert x_{ti} \Vert_{V_{t-1}^{-1}}$ for all $i \in [N]$
    \STATE Offer $S_{t} = \argmax_{S \subset \mathcal{S}}\Tilde{R}_{t}(S)$ in (\ref{equ: optimistic revenue}) and observe user feedback vector $y_{t}$
    \STATE Compute $V_{t} = \texttt{PrivateCov}(t,\rho_{2}) + 2\lambda I$ via Algorithm \ref{alg: privatecov}
    \IF {$\text{det}(V_{t}) > 2 \text{det}(V_{t-1})$ and $D_{\text{MLE}} < D$}
    \STATE $\hat{\theta}_{t} = \texttt{PrivateMLE}(t,\frac{\rho_{1}}{D})$, $D = D+1$
    \ELSE
    \STATE $\hat{\theta}_{t}=\hat{\theta}_{t-1}$
    \ENDIF
    \ENDFOR
\end{algorithmic}
\end{algorithm}

The algorithm stars from a pure exploration phase. During there \(T_0\) exploration periods, assortments are randomly selected to ensure reasonable initialization of parameter estimation. As shown in Thorem \ref{thm: regret}, \( T_0\) can be chosen as  \( \mathcal{O}\big(\frac{1}{\rho_1} \log T\big) \) to ensure reliable MLE. Since the exploration length scales logarithmically with \( T \), the contribution of \( T_0 \) to the overall regret is negligible in \( \Tilde{O} \)-notation. In addition, the choice of \( T_0 \) is also influenced by the privacy parameter. A smaller value of \( \rho_1 \) increases the privacy requirement in \texttt{PrivateMLE}, introducing a greater perturbation. This necessitates a larger \( T_0 \) to collect a sufficiently strong signal from the explored information. It is important to note that \( T_0 \) affects only MLE estimation and remains independent of the privacy parameter \( \rho_2 \) related to \texttt{PrivateCov}.

After the \( T_0 \) pure exploration phases, the algorithm transitions into the exploration-exploitation phase, where it uses a perturbed UCB strategy to select assortments based on the perturbed utility estimate for each item:
\[
    z_{ti} = x_{ti}^\top \hat{\theta}_{t-1} + \beta_{ti},
\]
where the first term, \( x_{ti}^\top \hat{\theta}_{t-1} \), represents the estimated utility of item \( i \) for user \( t \). Here, \( \hat{\theta}_{t-1} \) is a perturbed MLE obtained through \texttt{PrivateMLE} with privacy parameter \( \rho_{1}/D_{\text{MLE}} \). The second term, \( \beta_{ti} \), which refers to as \textit{exploration bonus}, quantifies the uncertainty in the utility estimate and is calculated as \( \alpha_t \Vert x_{ti} \Vert_{V_{t-1}^{-1}} \), where \( \alpha_t \) is selected based on the confidence bound of \( \hat{\theta}_{t} \), as established in Lemma \ref{thm: mle bound matrix norm}. The matrix \( V_{t-1} \), representing the shifted noisy Gram matrix, is computed as \texttt{PrivateCov}\((t-1, \rho_{2}) + 2\lambda I\), where the \texttt{PrivateCov} subroutine will be discussed in Algorithm \ref{alg: privatecov}. The parameter \( \lambda \) is chosen to ensure that \( V_{t-1} \) is positive definite; see Lemma \ref{thm: pd noisy gram} for the explicit formulation of \( \lambda \).


Both privacy parameters, \( \rho_1 \) and \( \rho_2 \), affect the confidence width \( \alpha_t \). When \( \rho_1 \) decreases, more perturbation in \texttt{PrivateMLE} leads to a higher estimation error in $\hat{\theta}_{t}$, which results in a larger \( \alpha_t \) to account for the added uncertainty. On the other hand, as \( \rho_2 \) decreases, more noise is added to the gram matrix, which requires a larger \( \lambda \). Since \( \alpha_t \) measures the estimation error in a weighted norm with respect to \( V_{t-1} \), a larger \( \lambda \) also leads to a larger \( \alpha_t \).

The hyperparameter \( D_{\text{MLE}} \) in Algorithm \ref{alg: dpmnl} limits the frequency of MLE updates. Frequent updates would introduce excessive noise, so we impose a determinant condition, \( \text{det}(V_t) > 2 \text{det}(V_{t-1}) \), ensuring updates occur only when sufficient new information is gathered. This condition restricts the number of updates to \( \mathcal{O}(d \log KT) \), and thus, \( D_{\text{MLE}} =  \mathcal{O}(d \log KT) \), thereby balancing performance and privacy by limiting unnecessary noise accumulation.

The UCB-based optimistic revenue estimate is:
\begin{equation}\label{equ: optimistic revenue}
    \Tilde{R}_t(S) = \frac{\sum_{i \in S} r_{ti} \exp(z_{ti})}{1 + \sum_{j \in S} \exp(z_{tj})}.
\end{equation}
As demonstrated in Lemma \ref{thm: oh lemma 3_5}, this estimate serves as an upper bound for the true expected revenue of the optimal assortment. 
By selecting assortments based on the perturbed optimistic utility \( z_{ti} \), the algorithm ensures both privacy and personalization.

\subsection{Intuivie Overview on Perturbed UCB Strategy}

Our perturbed UCB strategy addresses the exploration-exploitation tradeoff in a privacy-preserving manner by introducing calibrated noises into the utility estimates. The utility for user \(t\) and item \(i \in [N]\) in the UCB algorithm is given by:
\[
    z_{ti} = \underbrace{x_{ti}^\top \hat{\theta}_{t-1}}_{\text{point estimate}} + \underbrace{\alpha_{t} \Vert x_{ti} \Vert_{V_{t-1}^{-1}}}_{\text{exploration bonus}},
\]
where \(\hat{\theta}_{t-1}\) is the perturbed MLE estimate of the true parameter \(\theta^\ast\), and the exploration bonus is essential for balancing exploration and exploitation.

Compared to the non-private UCB, both the point estimate and the exploration bonus are perturbed due to privacy constraints, thus incurring a ``privacy cost." The point estimate, for instance, requires a longer exploration phase \(T_0\) due to smaller values of \(\rho_1\), which demands more time to gather sufficient information for reliable estimates. Additionally, the MLE itself has a higher estimation error compared to its non-perturbed counterpart. This is reflected in the confidence bound \(\alpha_t\), which increases as \(\rho_1\) decreases, showing that there is a larger uncertainty in the point estimate compared to the non-private setting.

The exploration bonus plays a crucial role in balancing the trade-off between exploration and exploitation by allocating a smaller bonus to well-explored items and a larger bonus to under-explored ones. To demonstrate this, we decompose \( V_{t-1} \) as follows:
\[
V_{t-1} = \Sigma_{t-1} + N_{t-1} + 2 \lambda I,
\]
where \(\Sigma_{t-1} = \sum_{n=1}^{t-1}\sum_{i \in S_{n}}x_{ni}x_{ni}^\top\) is the non-perturbed Gram matrix, and \(N_{t-1}\) is the noisy matrix added by \texttt{PrivateCov} (Algorithm \ref{alg: privatecov}) to satisfy \(\rho_2\)-zCDP. Further spectral decomposition yields that 
\[
\alpha_{t} \Vert x_{ti} \Vert_{V_{t-1}^{-1}} = \alpha_{t}  \sqrt{x_{ti}^\top \left(\Sigma_{t-1}+N_{t-1}+2\lambda I \right)^{-1}x_{ti}} = \alpha_{t}  \sqrt{\frac{(x_{ti}^\top v_{1})^{2} }{\lambda_{1}+\eta} + \cdots \frac{(x_{ti}^\top v_{d})^{2} }{\lambda_{d}+\eta}},
\]
where \(v_{i}\) are the eigenvectors of \(V_{t-1}^{-1}\), \(\lambda_{1} \geq \lambda_{2} \geq \cdots \lambda_{d} > 0\) are the eigenvalues of \(\Sigma_{t-1}\), and \(\eta\) represents the eigenvalues of the perturbed matrix \(N_{t-1} + 2\lambda I\).

This decomposition explicitly reveals the cost of privacy. Without the injected noise, the context vector \(x_{ti}\) would project onto the eigenvectors of the true Gram matrix \(\Sigma_{t-1}\), scaled by the eigenvalues. For well-explored context vectors, projections onto larger eigenvectors would result in smaller exploration bonuses, as these are scaled by larger eigenvalues. However, since we use the perturbed matrix \(V_{t-1}\), the projections are instead taken onto the eigenvectors of the noisy matrix to protect user's privacy. Furthermore, the scaling by \(\lambda_{i} + \eta\) inflates the exploration bonus uniformly, particularly when \(\lambda_i\) is smaller than \(\eta\), reducing the algorithm's precision in distinguishing between well-explored and under-explored directions.

Additionally, the confidence width \(\alpha_t\), which depends on both \(\rho_1\) and \(\rho_2\), prolongs the exploration phase, potentially leading to extended periods of suboptimal decision-making and a higher regret. Nevertheless, as exploration progresses, the exploration bonus diminishes, eventually converging to zero as \(\lambda_i\) grows with increasing \(t\). Once sufficient exploration has been conducted, the perturbed MLE closely approximates the true parameter \(\theta^\ast\), enabling effective exploitation of the learned preferences. Consequently, we can expect the regret to approach the rate seen in non-private settings in sufficiently large time horizon $T$.

\subsection{Design and Analysis of Privacy Algorithms}

In this section, we provide detailed explanations on the two subroutines \texttt{PrivateMLE} and \texttt{PrivateCov} used in Algorithm \ref{alg: dpmnl}. 

\subsubsection{Private MLE via Objective Perturbation}
We derive a differentially private MLE using the objective perturbation framework. The objective perturbation algorithm solves: 
\begin{equation}\label{equ: objective perturbation} 
\hat{\theta}(Z) = \argmin_{\theta \in \Theta}\mathcal{L}(\theta; Z,b) = \argmin_{\theta \in \Theta} \sum_{z \in Z} l(\theta;z) + \frac{\Delta}{2}\Vert \theta \Vert_{2}^{2} + b^\top \theta, 
\end{equation} 
where $b \sim N(0,\sigma^{2} I)$ and the parameters $\Delta$ and $\sigma^{2}$ are chosen to satisfy the privacy guarantee.

Objective perturbation is a well-established technique for satisfying DP in optimization problems, yet existing methods, whether for \((\epsilon,\delta)\)-DP \citep{chaudhuri2011differentially, kifer2012private} or \(\rho\)-zCDP \citep{redberg2024improving}, are not directly applicable to our multinomial model due to their assumption that the hessian of the objective function has a rank of at most 1. This condition, applicable in generalized linear models, does not hold in the multinomial logit model, where the Hessian has a higher rank. In addition to this limitation, current methods also diverge from the requirements of our setting. The approaches developed by \cite{chaudhuri2011differentially} and \cite{kifer2012private}, while satisfying \((\epsilon,\delta)\)-DP, do not satisfy  \(\rho\)-zCDP that we aim to achieve. Furthermore, the method proposed by \cite{redberg2024improving}, though aligned with \(\rho\)-zCDP, is formulated for unbounded DP, whereas we require bounded DP.

To address these limitations, we develop a novel objective perturbation algorithm that satisfies bounded \(\rho\)-zCDP and is applicable to multinomial models.

\begin{theorem}\label{thm: objective perturbation}(Bounded $\rho$-zCDP guarantee of objective perturbation). Let $\ell(\theta;z)$ be convex and twice-differentiable with $\Vert \nabla \ell(\theta, z)\Vert_{2} \leq L$ and the eigenvalues of $\nabla^{2}\ell(\theta;z)$ is upper bounded by $\eta$ for all $\theta \in \Theta$ and $z \in Z$. In addition, let $R$ be the rank of the hessian matrix of $l(\theta;z)$. Then the objective perturbation in Equation (\ref{equ: objective perturbation}) satisfies $\rho$-zCDP when 
$\Delta \geq \frac{\eta}{\exp(\frac{(1-q)\rho}{R})-1}$ and $\sigma^{2} \geq \left(\frac{L\left(\sqrt{d+2q\rho}+\sqrt{d}\right)}{q\rho}\right)^{2}$, for any $q \in (0,1)$.
\end{theorem}

Theorem \ref{thm: objective perturbation} allows for a general rank $R$, unlike previous methods \citep{chaudhuri2011differentially,kifer2012private,redberg2024improving} that assume $R \leq 1$. This makes it applicable to multinomial models, where the rank $R = \min \{d,K-1\}$ of their hessian matrices is usually greater than 1. 
The parameters \(L\), \(\eta\), and the rank \(R\) are used to bound the R\'enyi divergence by controlling the influence of individual data entries. Moreover, Theorem \ref{thm: objective perturbation} is of independent interest, as it extends differential privacy to a wider class of convex functions, enabling private optimization for more complex models beyond previous rank constraints.


The main technical challenge in proving Theorem \ref{thm: objective perturbation} is that we cannot directly apply the commonly used additive noise mechanisms, which typically inject noise proportional to the sensitivity—the maximum change in the output when a single input is altered. These methods achieve differential privacy by obscuring the contribution of individual data points. However, in our case, the complexity of the optimization problem makes it difficult to compute the sensitivity. To overcome this limitation, we leverage an alternative approach based on the first-order optimality condition $ b(a;Z) = -\big(\nabla \sum_{z \in Z} l(\theta;z) + \Delta^\top \theta\big)$
where \( a = \argmin_{\theta \in \Theta} \left(\sum_{z \in Z} l(\theta;z) + \frac{\Delta}{2}\Vert \theta \Vert_{2}^{2} + b^\top \theta\right) \). We then verify the variance of such gaussian noise \( b \) to establish the necessary bounds on the R\'enyi divergence. This method enables us to avoid the need for explicitly calculating the sensitivity.

\begin{algorithm}[htb!]
\caption{\texttt{PrivateMLE}}\label{alg: private mle}
\begin{algorithmic}[1]
    \STATE \textbf{Input:} current time step $t$, dimension of context $d$, assortment size $K$, rank of the hessian matrix $R = \min\{d,K-1\}$, privacy parameter $\rho$
    \STATE Fix any $q \in (0,1)$
    \STATE Set $\Delta = 4\left(\exp\left(\frac{(1-q)\rho_{1}}{RD}\right)-1\right)^{-1}, \sigma_{\text{MLE}}^{2} = \left(\frac{2(\sqrt{d+2q\rho}+\sqrt{d})}{q\rho}\right)^{2}$
    \STATE Sample $b \sim N(0,\sigma_{\text{MLE}}^{2}I)$
    \STATE \textbf{Output: } $\hat{\theta}_{t} =
    \sum_{n=1}^{t}\sum_{i \in S_{n}\cup\{0\}}-y_{ti}\log p_{n}(i \vert S_{n},\theta)+ \frac{\Delta}{2}\Vert \theta \Vert_{2}^{2} + b^\top \theta$
\end{algorithmic}
\end{algorithm}

Algorithm \ref{alg: private mle} gives the tailored version of privatized MLE in our multinomial model by leveraging Theorem \ref{thm: objective perturbation}. In our model, the parameters from Theorem \ref{thm: objective perturbation} can be computed as $L=2,\eta=4$ and $R=\min\{d,K-1\}$, as shown in Section \ref{proof: privatemle}. Moreover, we fix $q=1/2$ in all the numerical studies in Section \ref{sec: numerical study}.
Although Algorithm \ref{alg: private mle} is presented with a generic privacy parameter \(\rho\), its implementation within the \texttt{DPMNL} framework uses a privacy parameter of \(\frac{\rho_{1}}{D_{\text{MLE}}}\), where \texttt{PrivateMLE} is executed at most \(D_{\text{MLE}}\) times.

\begin{corollary}\label{thm: mle privacy}
    Each call of \texttt{PrivateMLE} (Algorithm \ref{alg: private mle}) with privacy parameter $\rho = \frac{\rho_{1}}{D_{\text{MLE}}}$ satisfies $\frac{\rho_{1}}{D_{\text{MLE}}}$-zCDP. Suppose \texttt{PrivateMLE} is called at most $D_{\text{MLE}}$ times in Algorithm \ref{alg: dpmnl}. Then the composition of $D_{\text{MLE}}$ outputs satisfies $\rho_{1}$-zCDP.
\end{corollary}

\subsubsection{Private Cov via Tree-based Aggregation}

\texttt{PrivateCov} in Algorithm \ref{alg: privatecov} is designed to continuously output a sequence of noisy gram matrices while controlling the noise accumulation over time. Given the set $S_{t}$ and the context vectors $\{x_{ti}\}_{i \in S{t}}$, Algorithm \ref{alg: privatecov} continuously updates a binary tree structure to efficiently update and release privatized Gram matrices. 

\begin{algorithm}[htb!]
\caption{\texttt{PrivateCov}}\label{alg: privatecov}
\begin{algorithmic}[1]
\STATE \textbf{Input: } privacy parameters $\rho_{2}$, $T$
\STATE Set $\sigma_{\text{Cov}}^{2} = \frac{Km}{\rho_{2}}$, $m = 1+\lfloor \log_{2}T \rfloor$ 
\STATE \textbf{Initialization: } $p(l) = \hat{p}(l) = 0$ for all $l = 0,\cdots,m-1$
\FOR{$t \in [T]$}
\STATE Express $t$ in binary form: $t = \sum_{l=0}^{m-1} \text{Bin}_{t}(l)2^{l}$, where $\text{Bin}_{t}(l) \in \{0,1\}$
\STATE Find index of first one $l_{n} = \min \{l: \text{Bin}_{t}(l)=1\}$
\STATE Update $p$-sums: $p(l_{n}) = \sum_{l < l_{n}}p(l) + \sum_{i \in S_{t}}x_{ti}x_{ti}^\top$ and $p(l) = \hat{p}(l) = 0$ for all $l < l_{n}$
\STATE Inject noise: $\hat{p}(l_{n}) = p(l_{n}) + N \text{ where }N_{ij} = N_{ji} \stackrel{\text{i.i.d.}}{\sim} \mathcal{N}(0,\sigma_{\text{Cov}}^{2}) $
\STATE Release $V_{t} = \sum_{l=0}^{m-1} \text{Bin}_{t}(l)\hat{p}(l)$
\ENDFOR
\end{algorithmic}    
\end{algorithm}

The tree-based aggregation technique \citep{chan2011private} constructs a complete binary tree in an online manner, where each node contains a partial sum combined with independent noise. The leaf nodes correspond to individual data entries, while internal nodes represent aggregated sums of their respective subtrees. The depth of this binary tree is at most \(\lfloor \log_{2} t \rfloor + 1\), which provides two significant advantages in terms of DP and utility.

First, the sensitivity of the noise introduced during the aggregation process, defined as the maximum change in output due to the modification of a single user's data, is constrained by the limited number of affected nodes, bounded by \(\lfloor \log_2 t \rfloor + 1\). This logarithmic bound on sensitivity ensures that the amount of noise required to maintain privacy does not scale linearly with the total number of users \(T\), as would occur in a naive aggregation approach where every user's data affects the output at every step.

Second, at each update at time \(t\), the published output is the sum of at most \(\lfloor \log_2 t \rfloor + 1\) nodes. This constraint limits noise accumulation to a small subset of the total nodes. This controlled aggregation ensures that noise accumulation grows logarithmically with \(t\), preserving data utility while satisfying privacy requirements. If noise were allowed to accumulate unrestricted across all nodes, the resulting degradation in data quality would significantly impair the algorithm’s performance.

By employing the tree-based aggregation technique for Gram matrices, sum statistics over contextual vectors, \texttt{PrivateCov} efficiently privatizes the gram matrices.

\begin{remark}\label{remark: tree and linear bandit}(Comparison with Differentially Private Linear Contextual Bandits). In private linear contextual bandits, the tree-based aggregation technique is widely employed due to its compatibility with the sufficient statistics of UCB. For context vectors \(x_{n} \in \mathbb{R}^{d}\) and feedbacks \(y_{n} \in \mathbb{R}\), these sufficient statistics can be expressed as \((\sum_{n=1}^{t}x_{n}x_{n}^\top, \sum_{n=1}^{t}x_{n}y_{n})\). The least-squares estimator, \((\sum_{n=1}^{t}x_{n}x_{n}^\top)^{-1}\sum_{n=1}^{t}x_{n}y_{n}\), fits naturally into the tree aggregation framework, enabling a straightforward implementation of DP linear contextual bandits without the need for additional privacy mechanisms \citep{shariff2018differentially}. In contrast, the MLE or its sufficient statistics in our multinomial logit setting cannot generally be expressed as simple sum statistics, making the tree-based technique unsuitable for directly privatizing the MLE. To overcome this limitation, we adopt the asynchronous update strategy introduced in Algorithm \ref{alg: dpmnl}, which mitigates noise accumulation in MLE updates by limiting the frequency of updates to \(O(\log T)\). This highlights the new technical difficulty to achieve our privacy-preserving MNL bandit policy.
\end{remark}

\begin{theorem}\label{thm: Cov DP guarantee}
    \texttt{PrivateCov} satisfies \(\rho_{2}\)-zCDP.
\end{theorem}

Theorem \ref{thm: Cov DP guarantee} leverages the variance of Gaussian random matrices in Algorithm \ref{alg: privatecov}. In our multinomial model with \(K\) items, the noise level scales as \(\sigma_{\text{Cov}}^{2} = \frac{Km}{\rho_{2}}\), which is driven by the sensitivity of the covariance matrix, quantifying the maximum change in output when a single user’s data is modified. Lemma \ref{thm: sensitivity} establishes that the sensitivity in our setting is \(\sqrt{2K}\). This contrasts with linear contextual bandits or GLM bandits, which have a sensitivity of \(\sqrt{2}\) under the assumption \(\Vert x \Vert_{2} \leq 1\) \citep{biswas2020coinpress}, since only one “category” of information per context is involved. The higher sensitivity in our multinomial model reflects the added complexity of capturing multiple item interactions within each assortment, requiring proportionally more noise to guarantee privacy.

Finally, we provide the privacy guarantee of Algorithm \ref{alg: dpmnl}.

\begin{theorem}\label{thm:jdp guarantee} Given the $D_{\text{MLE}}$ composition of \texttt{PrivateMLE} satisfies $\rho_{1}$-zCDP and \texttt{PrivateCov} satisfies $\rho_{2}$-zCDP respectively, \texttt{DPMNL} satisfies $(\rho_{1}+\rho_{2})$-Joint zCDP.
\end{theorem}

\begin{remark}(Privacy considerations during the initial exploration phase).
    During the pure-exploration phase until \(T_0\), each assortment is selected randomly without relying on historical data or individual user contexts. Therefore, there is no risk of privacy leakage from the observed output sequence during this period. Once \(T_0\) is reached, although user-specific information begins to influence the assortments, it is masked by the privacy-preserving subroutines. As a result, all users remain protected, ensuring that their personal information is not exposed throughout the algorithm's operation.
\end{remark}

\section{Regret Analysis}\label{sec: regret analysis}
Regret analysis evaluates the performance of bandit policies by comparing their outcomes to those of an optimal policy with full knowledge of the environment. In our setting, regret measures the cumulative difference between the expected reward of the optimal assortment and that achieved by our privacy-preserving policy, \texttt{DPMNL}. Unlike non-private MNL bandits, our analysis also incorporates the influence of privacy parameters, highlighting the trade-off between privacy and performance.

The regret analysis is based on the following two assumptions that commonly appears in MNL bandit literature.

\noindent
\begin{assumption}\label{ass: sigma zero}
Each feature vector $x_{ti}$ is drawn $i.i.d$ from an unknown distribution $p_{x},$ with $\Vert x_{ti} \Vert \leq 1$ for all $t,i$ and there exists a constant $\sigma_{0} > 0$ such that $\lambda_{\text{min}}\left(\mathbb{E}[x_{ti}x_{ti}^\top]\right) \geq \sigma_{0}.$
\end{assumption}

\begin{assumption}\label{ass: kappa}
There exists $\kappa > 0$ such that for every item $i \in S$ and any $S \in \mathcal{S}$ and all round $t$, $\min_{\Vert \theta - \theta^{\ast}\Vert \leq 1} p_{t}(i\vert S,\theta)p_{t}(0 \vert S,\theta) \geq \kappa.$     
\end{assumption}

Assumption \ref{ass: sigma zero} implies that consumer's context vectors are relatively widely spread and not concentrated in a specific direction. Furthermore, the increase in the minimum eigenvalue of the associated gram matrix results in the convergence of parameter estimation. The $i.i.d$ assumption is common in generalized linear bandit \citep{li2017provably} and MNL contextual bandit \citep{article,oh2019thompson, oh2021multinomial}. Assumption \ref{ass: kappa} is used to guarantee the asymptotic normality of MLE, especially to ensure the Fisher information matrix to be invertible \citep{lehmann2006theory}. This is a standard assumption in MNL contextual bandits \citep{article,chen2020dynamic,oh2019thompson,oh2021multinomial}, which is also equivalent to the standard assumption for the link function in generalized linear contextual bandits \citep{filippi2010parametric, li2017provably}.

\begin{theorem}\label{thm: regret}
Under Assumptions \ref{ass: sigma zero} and \ref{ass: kappa}, the expected cumulative regret of our \texttt{DPMNL} with $T_{0} = \frac{1}{K}\left( \frac{C_{1}\sqrt{d}+C_{2}\sqrt{2\log T}}{\sigma_{0}} \right)^{2} + \frac{2C_{\rho_1,T}}{K \sigma_{0}}$, where $C_{1}$ and $C_{2}$ are some constants and $C_{\rho_{1},T}$ is defined in Lemma \ref{thm: mle bound l2 norm}, can be bounded as: 
\[
\mathcal{R}_{T} \leq \Tilde{O}\left(\left(d+\frac{d^{5/2}}{\rho_{1}}+\frac{d^{3/4}}{\rho_{2}^{1/4}}\right)\sqrt{T}\right).
\]
\end{theorem}

Theorem \ref{thm: regret} provides a simplified regret bound for our privacy-preserving algorithm. The explicit exact upper bound of $\mathcal{R}_{T}$ is provided in Equation (\ref{equ: regret}) of Proof of Theorem \ref{thm: regret}. For MNL contextual bandits with a \(d\)-dimensional contextual vector, \cite{chen2020dynamic} established an \(\Omega(d\sqrt{T}/K)\) lower bound on the cumulative regret, where \(K\) is the maximum number of displayed items. Treating \(K\) as constant, this simplifies to \(\Omega(d\sqrt{T})\), and our result matches this optimal rate in \(T\), confirming that our algorithm achieves near-optimal performance despite the privacy constraints.

The regret bound further reflects the impact of the privacy parameters \(\rho_1\) and \(\rho_2\). The term \(d^{5/2}/{\rho_1}\) indicates that \texttt{PrivateMLE} has a stronger influence on the regret than that of \texttt{PrivateCov}, which contributes \(d^{3/4}/{\rho_2^{1/4}}\). This suggests that a greater share of the privacy budget should be allocated to \texttt{PrivateMLE}, as the accuracy of the MLE becomes more crucial in later stages when exploration diminishes and exploitation becomes dominant.

Finally, the dependency on \(d\) is stronger in our private setting due to the noise added to mask the sensitive information contained in the contextual vectors. The need to ensure JDP requires noise that scales with the dimension of the context vector, leading to a larger influence of \(d\) on the regret.

\begin{remark}
    It is remarkable that the regret in Theorem \ref{thm: regret} has \emph{no} dependency on the total number of items $N$. This is because the context information is used via $d \times d$ gram matrices and the noise is injected into the gram matrices. In this regard, a naive local DP algorithm that utilizes the perturbed context for all $N$ items would result in the regret that \emph{depends} on $N$. It is prohibitive as $N$ is typically large in the real application of assortment selection.
\end{remark}


\section{Numerical Study}
\label{sec: numerical study}

In this section, we evaluate the cumulative regret of our \texttt{DPMNL} policy across varying scenarios, including different privacy parameters $\rho_{1},\rho_{2}$, dimension of contextual vectors $d$, and capacity of the assortment $K$. We begin by examining it using synthetic data
in Section \ref{sec: simulation} and subsequently validate its practicability in a real-world dataset in Section \ref{sec: real data}.

\subsection{Synthetic Data Analysis}\label{sec: simulation}
In this section, we empirically examine the trade-off between privacy parameters and cumulative regret, as well as the effect of privacy budget allocation as outlined in Theorem \ref{thm: regret}. Additionally, we investigate the dependencies on the dimension of the contextual vector \(d\) and the assortment size \(K\). Since this work introduces the first privacy-preserving MNL contextual bandit, there is no direct comparison available in the existing literature. As a result, we construct a benchmark by extending the privacy-preserving GLM bandit approach studied in \cite{chen2022privacy}, adapting it to satisfy \((\epsilon, \delta)\)-JDP in the multinomial model. A significant challenge in adapting existing methods for generalized linear models is that there is no existing objective perturbation result for multinomial model that satisfies \((\epsilon, \delta)\)-DP. To address this, we develop such a mechanism, detailed in Theorem \ref{thm: objective perturbation eps del} of the Appendix. This adaptation allows for a fair comparison between our \(\rho\)-zCDP-based definition and the benchmark \((\epsilon,\delta)\)-JDP policy.

Throughout the experiment, the total number of items is fixed at \(N = 100\), and we assign uniform revenues \( r_{ti} = 1 \) for all items \( i \) and time steps \( t \). This setup aligns with the objective of maximizing the online click-through rate, as discussed in \cite{oh2021multinomial}. Additionally, we use \( c\alpha_{t} \) in place of \( \alpha_{t} \), with \( c = 10^{-4} \), to uniformly reduce the magnitude of the exploration bonus across all items. This adjustment enables the observation of asymptotic behavior even within a shorter time horizon. The dimension of the contextual vector, \( d \), is set to \( d = 5 \) unless we are specifically investigating dependency on \( d \). The contextual vectors are generated from an i.i.d. multivariate Gaussian distribution \( N(0, I_{d}) \), and the true parameter \( \theta^{\ast} \) is sampled from a uniform distribution over \([0, 1]^d\). Similarly, the assortment size \( K \) is set to \( K = 10 \) unless \( K \) is being varied to examine its dependency. For each experimental configuration, we perform 30 independent replicates and report the average cumulative regret along with the corresponding confidence bands.
\begin{figure}[ht]
    \centering
    \includegraphics[width=\textwidth]{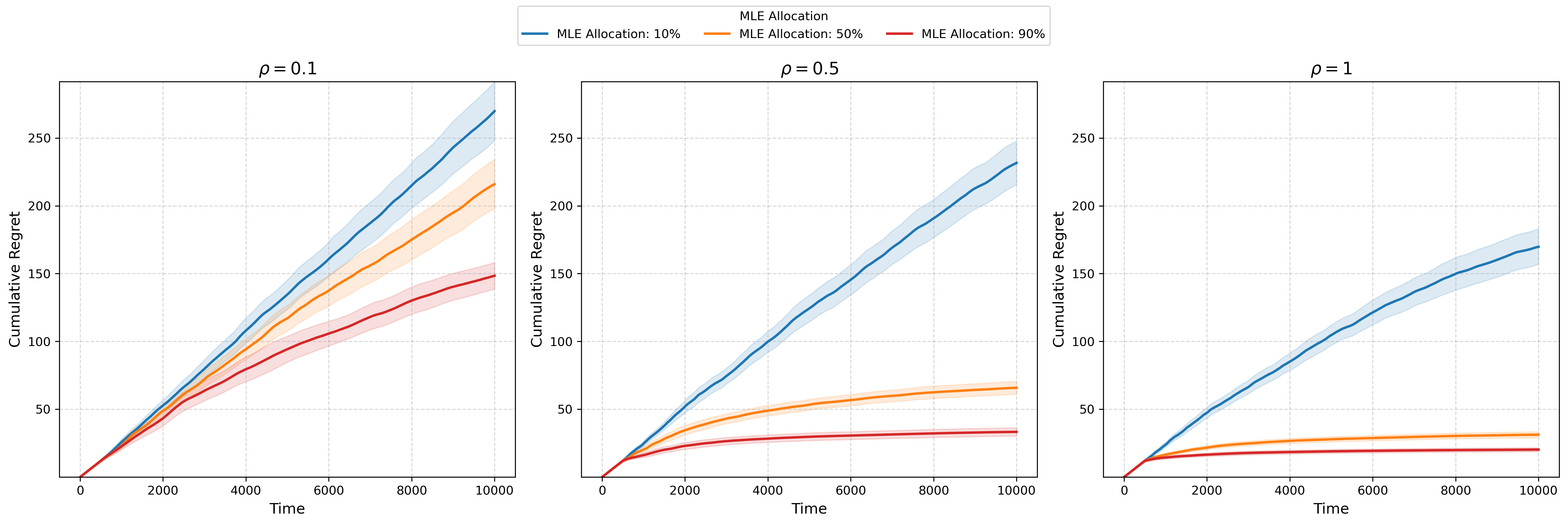}
    \caption{Comparison of cumulative regrets across different level of privacy parameter \( \rho \in \{0.1, 0.5, 1\}\) with each line representing a different allocation of \( \rho \) between \texttt{PrivateMLE} and \texttt{PrivateCov}. For instance, ``MLE Allocation: 90\%" indicates that 90\% of \( \rho \) is allocated to \texttt{PrivateMLE}, while the remaining 10\% is allocated to \texttt{PrivateCov}.}
    \label{fig: rho_comparison}
\end{figure}

In Figure \ref{fig: rho_comparison}, we observe two main findings that align with our theoretical predictions. First, as the privacy level decreases, or as the privacy parameter $\rho$ increases, cumulative regret decreases. This observation is consistent with the typical tradeoff between privacy and utility and supports our theory. Second, at the same privacy level, allocating a larger portion of the privacy budget to \texttt{PrivateMLE} yields better results. This finding also aligns with our theoretical insights, which suggest that \(\rho_{1}\) has a greater impact on the regret bound than \(\rho_{2}\).

\begin{figure}[ht]
    \centering
    \includegraphics[scale=0.43]{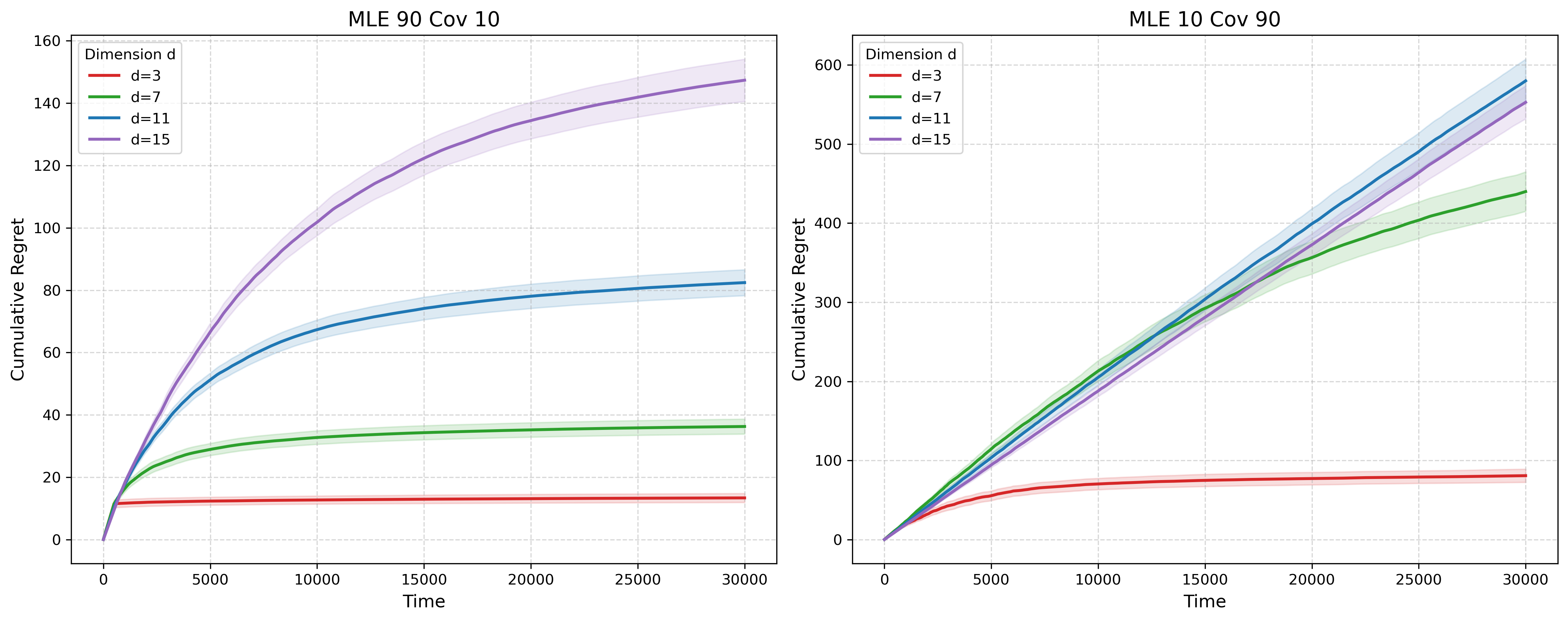}
    \caption{Effect of privacy budget allocation on cumulative regret across different dimensions.}
    \label{fig: simul_d}
\end{figure}

Figure \ref{fig: simul_d} demonstrates the cumulative regret trends for different dimensions of the contextual vector, specifically for \(d = 3, 7, 11,\) and 15, under varying allocations of the privacy budget between \texttt{PrivateMLE} and \texttt{PrivateCov}. The total privacy budget is fixed at \(\rho = 1\). As shown in the left plot of Figure \ref{fig: simul_d}, when 90\% of the privacy budget is allocated to \texttt{PrivateMLE}, the cumulative regret shows a clear sublinear growth across all dimensions, indicating a stabilization of the algorithm’s performance. This aligns with the theoretical upper bound \(R_{T} \leq \Tilde{O}\left(\left(d+\frac{d^{5/2}}{\rho_{1}}+\frac{d^{3/4}}{\rho_{2}^{1/4}}\right)\sqrt{T}\right)\) that allocating more of the privacy budget to \texttt{PrivateMLE}, or increasing \(\rho_{1}\), mitigates the dimensional impact on regret by reducing estimation errors. In contrast, as shown in the right plot of Figure \ref{fig: simul_d}, when only 10\% of the privacy budget is allocated to \texttt{PrivateMLE}, the cumulative regret exhibits a slower sublinear rate, requiring a longer time horizon to observe the sublinear trend for higher dimensions, $d=11,15$. These results empirically validate that a larger allocation to \texttt{PrivateMLE} effectively controls the dimensional impact on regret, thereby enhancing performance under privacy.

\begin{figure}[htbp]
    \centering
    \includegraphics[scale=0.4]{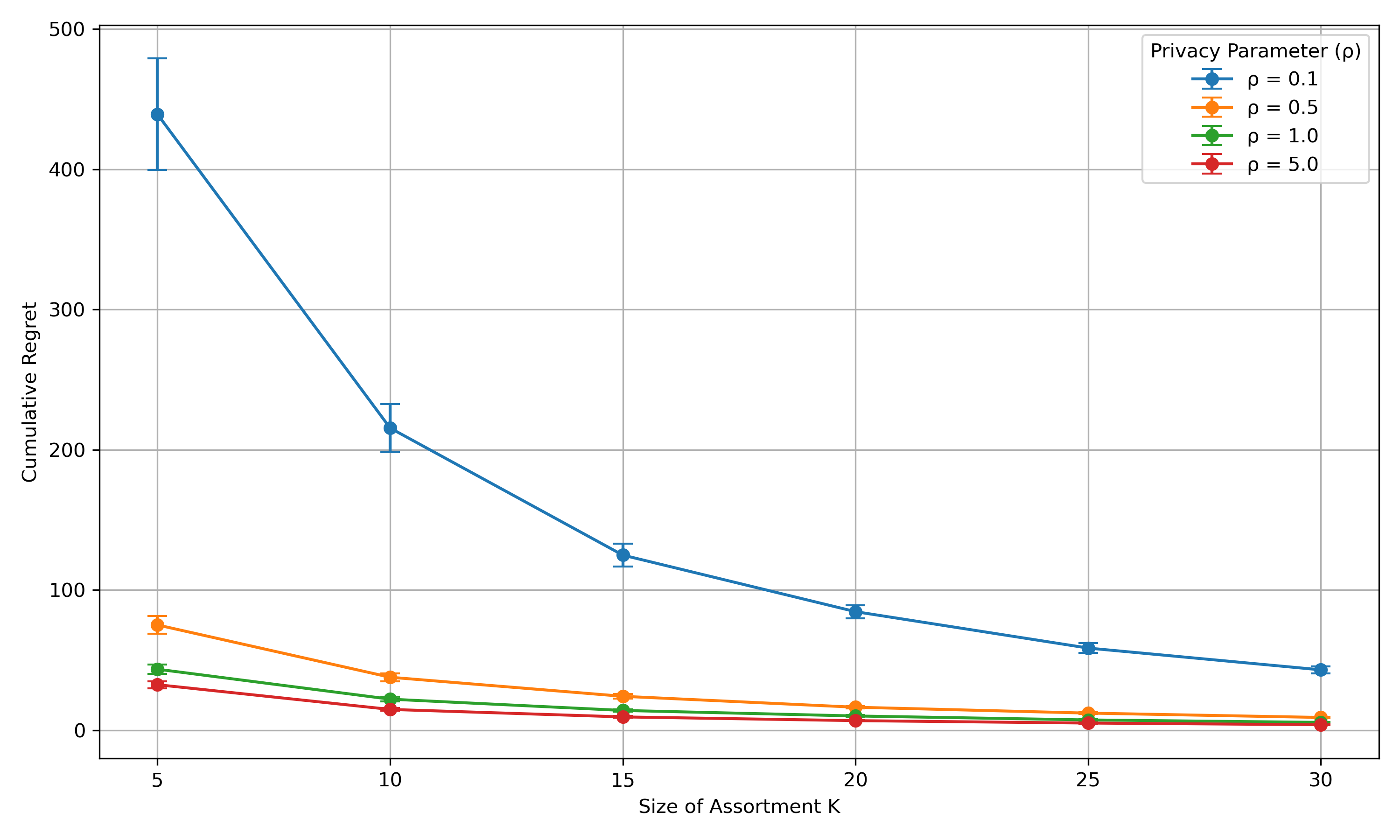}
    \caption{Impact of assortment size \(K\) on cumulative regret under different privacy levels \(\rho \in \{0.1,0.5,1,5\}\).}
    \label{fig: simul_K}
\end{figure}


Moreover, the empirical results presented in Figure \ref{fig: simul_K} demonstrate that increasing the assortment size \(K\) consistently reduces cumulative regret across various levels of the privacy parameter \(\rho\). This effect is most pronounced under stricter privacy settings \(\rho = 0.1\), where cumulative regret declines significantly as \(K\) increases. Note that a larger $K$ contributes to increased perturbation in both \texttt{PrivateMLE} and \texttt{PrivateCov}, as more items expose sensitive information about users, necessitating additional noise for privacy protection. At the same time, a larger \(K\) allows the algorithm to explore more items at a time, which can improve UCB estimates and lead to more accurate decision-making. Figure \ref{fig: simul_K} suggest that in this simulation while increasing \(K\) introduces more noise into the policy, the additional information provided by larger assortments outweighs this drawback, leading to improved performance.

\begin{figure}[htbp]
    \centering
    \includegraphics[width=\textwidth]{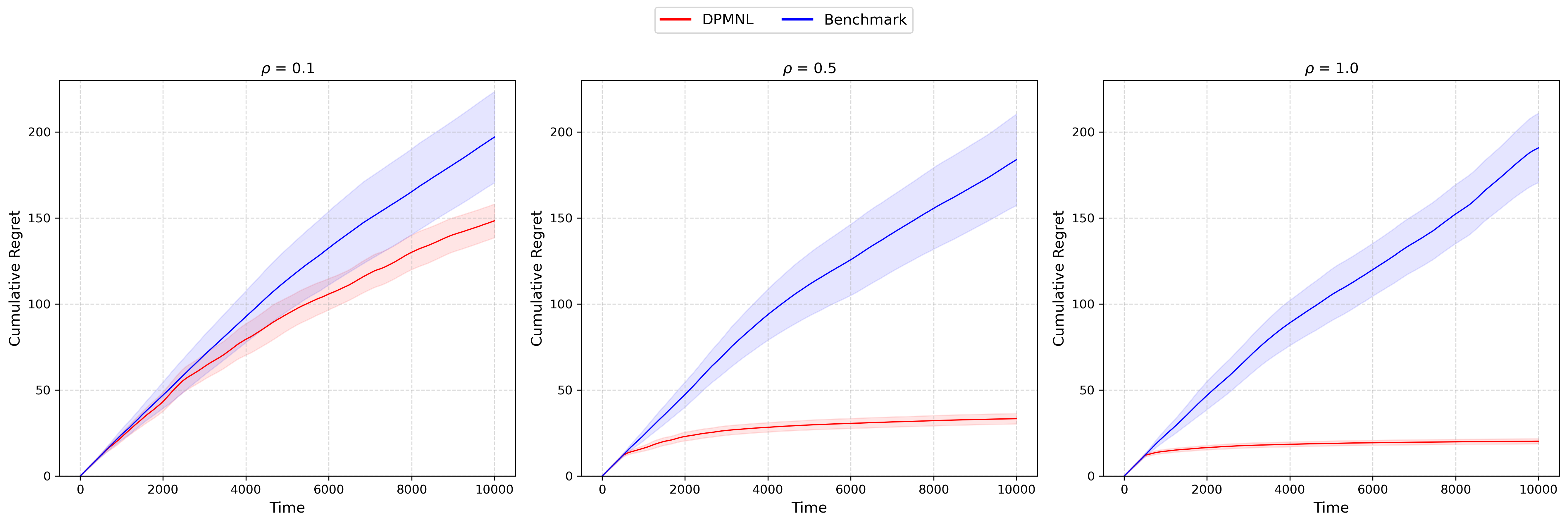}
    \caption{Comparison of cumulative regret between \texttt{DPMNL} and the benchmark across varying levels of privacy parameter $\rho \in \{0.1,0.5,1\}$ on synthetic data.}
    \label{fig: rho_epsdel_comparison}
\end{figure}

Figure \(\ref{fig: rho_epsdel_comparison}\) illustrates the empirical advantage of our \texttt{DPMNL} policy over the benchmark, which is an extension of \cite{chen2022privacy} to our setting. For a fair comparison, on the benchmark, we converted \(\rho\) values to \((\epsilon, \delta)\) using the relation \(\epsilon = \rho + 4\rho \log T\) and \(\delta = 1/T^2\), following Lemma \ref{thm: zcdp conversion}. In both cases, 90\% of the privacy budget was allocated to MLE. Compared to our \texttt{DPMNL}, this benchmark method shows higher cumulative regret, largely due to the looser composition properties of \((\epsilon, \delta)\)-DP, which result in increased injected noise. This highlights the limitations of the benchmark when applied on more complex multinomial model beyond generalized linear models.

\subsection{Real Data Analysis}\label{sec: real data}
In this section, we evaluate our algorithm on the ``Expedia Hotel" dataset \citep{expedia-personalized-sort} to assess its real-world performance. Protecting privacy in hotel booking data is critical, as these records can reveal sensitive details such as travel destinations, preferences, and financial status. Moreover, booking habits might expose personal routines like frequent business trips or vacations. Implementing private mechanisms can help prevent profiling, targeted marketing, and more serious risks like identity theft, thereby upholding ethical standards and fostering trust in the platform.

\subsubsection{Data Pre-processing}
The dataset consists of 399,344 unique searches on 23,715 search destinations each accompanied
by a recommendation of maximum 38 hotels from a pool of 136,886 unique properties. User
responses are indicated by clicks or hotel room purchases. The dataset also includes features for each
property-user pair, including hotel characteristics such as star ratings and location attractiveness,
as well as user attributes such as average hotel star rating and prices from past booking history.

We follow the pre-processing procedure in \cite{lee2024low} to address missing values and conduct feature transformation to aviod outliers. The details are provided in Section \ref{sec: processing expedia} of the Appendix. After pre-processing, we have $T = 4465$ unique searches encompassing  $N = 124$ different hotels, with $d_2 = 10$ hotel features and $d_1$ = 18 user features, where the description is given in Table 1. We normalize each feature to have mean 0 and variance 1. 

\begin{table}[ht]
    \centering
    \begin{tabular}{p{5cm} p{10cm}}
        \hline
        \vspace{0.4cm} \textbf{Item (Hotel) Features} \vspace{0.4cm} & Star rating (4 levels), Average review score (3 levels), Brand of property, Location score, Log historical price, Current price, Promotion \\
        \hline
        \vspace{0.6cm} \textbf{User (Search) Features} \vspace{0.6cm} & Length of stay (4 levels), Time until actual stay (4 levels), \# of adults (3 levels), \# of children (3 levels), \# of rooms (2 levels), Inclusion of Saturday night, Website domain, User historical star rating (3 levels), User historical price (4 levels) \\
        \hline
    \end{tabular}
    \caption{Table of Item and User Features.}
    \label{tab:features}
\end{table}

\subsubsection{Analysis of Expedia Dataset}
The experiment setup includes a scaled confidence width $c\alpha_{t}$ with \(c = 10^{-7}\) and a pure exploration phase of \(T_0 = 10,000\) during the total time horizion $T = 100,000$. Each plot of Figure \ref{fig: expedia} shows the average cumulative regret across 10 independent runs. In our implementation, we allocate 90\% of the privacy budget to \texttt{PrivateMLE} and 10\% to \texttt{PrivateCov}, based on our findings from synthetic data experiments. 

\begin{figure}[ht]
    \centering
    \includegraphics[width=\textwidth]{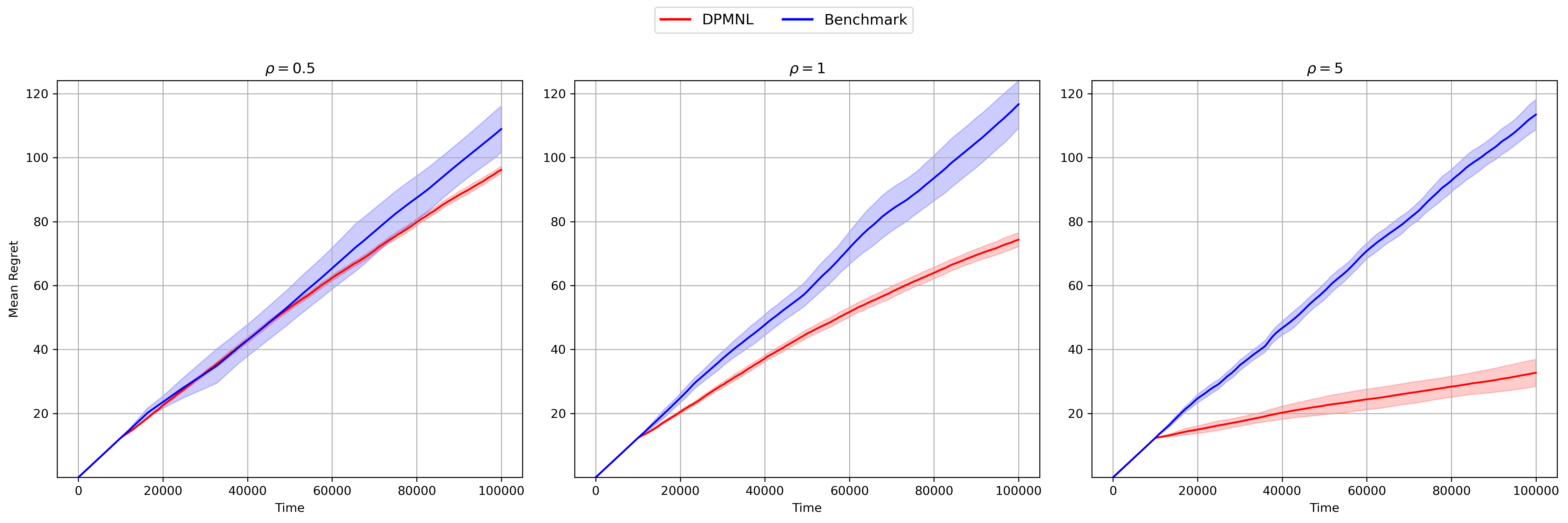}
    \caption{Comparison of cumulative regret between \texttt{DPMNL} and the benchmark across varying levels of privacy parameter $\rho \in \{0.5,1,5\}$ on Expedia dataset.}
    \label{fig: expedia}
\end{figure}

In Figure \ref{fig: expedia}, \texttt{DPMNL} consistently outperforms the benchmark across all \(\rho\) values, with the performance gap widening as \(\rho\) increases. For smaller \(\rho\), neither policy has yet exhibited a sublinear pattern in regret, suggesting that longer time horizon $T$ would be needed to further accentuate the differences. However, with larger privacy budgets (\(\rho = 1, 5\)), \texttt{DPMNL} demonstrates a clear sublinear regret, while the benchmark does not. \texttt{DPMNL} also exhibits significantly lower regret, confirming its superior effectiveness.

These results demonstrate that \texttt{DPMNL} performs robustly across both simulated and real-world data, underscoring its practicality. By leveraging \(\rho\)-joint zCDP in the bandit framework, \texttt{DPMNL} achieves better privacy-utility trade-offs than existing approaches, establishing it as a strong choice for privacy-preserving algorithms in personalized recommendation systems.

\baselineskip=16pt
\bibliographystyle{asa}
\bibliography{ref}

\newpage
\baselineskip=24pt
\setcounter{page}{1}
\setcounter{equation}{0}
\setcounter{section}{0}
\renewcommand{\thesection}{S.\arabic{section}}
\renewcommand{\thelemma}{S\arabic{lemma}}
\renewcommand{\theequation}{S\arabic{equation}}

\begin{center}
{\Large\bf Supplementary Materials} \\
\medskip
{\Large\bf ``Privacy-Preserving Dynamic Assortment Selection"}  \\
\bigskip
\textbf{Young Hyun Cho, Will Wei Sun}
\vspace{0.2in} 
\end{center}
\bigskip

\noindent
In this supplementary material, we provide detailed proofs and additional information to support the results presented in the main text. Specifically, Section \ref{sec: privacy proof} contains comprehensive proofs of our privacy analysis, including Lemma \ref{thm: billboard lemma}, Theorem \ref{thm: mle privacy}, \ref{thm: Cov DP guarantee} and \ref{thm:jdp guarantee}. In Section \ref{sec: regret proof}, we present the proof of Theorem \ref{thm: regret} regarding the regret analysis, along with key supporting lemmas. Section \ref{sec: simulation details} provides the simulation details and definitions used in our numerical studies. Finally, Section \ref{sec: technical lemma} includes technical lemmas that are instrumental in our analysis.

\section{Proofs of Privacy Analysis}\label{sec: privacy proof}
In this section, we provide detailed proofs of the privacy guarantees, starting with the proof of the billboard lemma in Lemma \ref{thm: billboard lemma}. While the billboard lemma is a widely-used tool for JDP \citep{hsu2014private, vietri2020private}, it has been presented in the \((\epsilon, \delta)\)-DP framework.

\begin{lemma}\label{thm: billboard lemma epsdel}
Suppose \( M: \mathcal{U}^{T} \rightarrow \mathcal{R} \) satisfies \((\epsilon,\delta)\)-DP. Consider any set of deterministic functions \( f_{t}:\mathcal{U}_{t} \times \mathcal{R} \rightarrow \mathcal{R}' \), where \(\mathcal{U}_{t}\) represents the data associated with the \(t^{th}\) user. The composition \(\mathcal{M}(U) = \{f_{t}(U_{t}, M(U))\}_{t=1}^{T}\) satisfies \(\epsilon,\delta\)-JDP, where \(U_{t}\) denotes the subset of the full dataset \(U\) corresponding to the \(t^{th}\) user’s data.
\end{lemma}

As our paper is the first that writes JDP in a $\rho$-zCDP framework, its validity in our $\rho$-zCDP has not been established before and so we prove that billboard lemma remains to hold in our new framework as well. Moreover, we prove the privacy guarantee on \texttt{PrivateMLE} and \texttt{PrivateCov}. On top of the guarantees, by leveraging the billboard lemma, we prove the privacy guarantee of our main policy \texttt{DPMNL}.

\subsection{Proof of Lemma \ref{thm: billboard lemma}: Billboard Lemma}
Let $U$ and $U'$ be the $t$-neighboring databases so that $\Pi_{j}U = \Pi_{j}U'$ for all $i \neq t$ and $M$ be a mechanism that satisfies $\rho$-zCDP. It is sufficient to show that for all $\alpha >1$,
\[
D_{\alpha}\left(\mathcal{M}_{-t}(U) \Vert \mathcal{M}_{-t}(U')\right) \leq \rho \alpha.
\]

We use the following property:
\begin{lemma}\label{thm: postprocessing}(Invariance to Post-processing\citep{bun2016concentrated})
Let $P$ and $Q$ be distributions on $\Omega$ and let $f: \Omega \rightarrow \Theta$ be a function. Let $f(P)$ and $f(Q)$ be the distributions on $\Theta$ induced by applying $f$ to $P$ or $Q$ respectively. Then
\begin{equation*}
D_{\alpha}\left(f(P) \Vert f(Q)\right) \leq D_{\alpha}\left(P \Vert Q\right). 
\end{equation*}    
\end{lemma}

Note then $\mathcal{M}_{-t}(U)$ can be regarded as $\mathcal{M}_{-t}(U) = f\left(M(U);U_{-t}\right)$ for some function $f$, where $U_{-t}$ is treated as constant since it is known for the adversaries. Therefore,
\begin{equation*}
    \begin{split}
        D_{\alpha}\left(\mathcal{M}_{-t}(U) \Vert \mathcal{M}_{-t}(U')\right) &= D_{\alpha}\left(f(M(U))\Vert f(M(U'))\right)\\
        &\leq D_{\alpha}\left(M(U)\Vert M(U')\right) \leq \rho\alpha,
    \end{split}
\end{equation*}
where the first inequality is by Lemma \ref{thm: postprocessing}.

Intuitively, $U_{-t} := \{\Pi_{j}U\}_{j \neq t} = \{\Pi_{j} U'\}_{j \neq t}$ is already a public knowledge for adversaries and does not contain the target user $t$. Therefore, adversaries can't gain further information on user $t$ by the transformation using $U_{-t}$.

\subsection{Proof of Theorem \ref{thm: mle privacy}: Objective Perturbation}

Recall the objective perturbation is
\begin{equation*}
    \mathcal{L}(\theta;\mathcal{Z},b) = \sum_{z \in \mathcal{Z}}\ell(\theta,z) + \frac{\Delta}{2}\Vert \theta \Vert_{2} + b^\top \theta,
\end{equation*} and $\hat{\theta}(Z)$ is the minimizer, where $b \sim N(0,\sigma^{2})$.

Let $\alpha > 1, q \in (0,1)$ be given and take $Z,Z'$ be neighboring datasets that share $T-1$ individuals but only differ in one individual. Without loss of generality, we take $T$-neighboring databases by $Z = (Z_{-T},z)^\top$ and $Z' = (Z_{-T},z')^\top$. Given the neighboring databases $Z$ and $Z'$, we aim to show 
\begin{equation*}
    D_{\alpha}\left(\hat{\theta}(Z) \Vert \hat{\theta}(Z')\right) \leq \rho \alpha,
\end{equation*}
where $\Delta \geq \frac{\eta}{\exp(\frac{(1-q)\rho}{R})-1}$ and $\sigma^{2} \geq \left(\frac{L\left(\sqrt{d+2q\rho}+\sqrt{d}\right)}{q\rho}\right)^{2}$, as specified in Theorem \ref{thm: mle privacy}.

Define $R(a) = \frac{\text{pdf}_{\hat{\theta}(Z)}(a)}{\text{pdf}_{\hat{\theta}(Z)}(a)}$, where $\text{pdf}_{\hat{\theta}(Z)}(\cdot)$ denotes the pdf of random variable $\hat{\theta}(Z)$ and $\Tilde{L}(a) = \log R(a)$. Then 
\begin{equation*}
\begin{split}
    D_{\alpha}\left(\hat{\theta}(Z) \Vert \hat{\theta}(Z')\right) \leq \rho \alpha &= \frac{1}{\alpha-1}\log \mathbb{E}_{a \sim \hat{\theta}}\left[(R(a))^{(\alpha-1)}\right] \\
    &= \frac{1}{\alpha-1}\log \mathbb{E}_{a \sim \hat{\theta}}\left[e^{(\alpha-1)\Tilde{L}(a)}\right]
\end{split}
\end{equation*}
Therefore, we focus mainly on getting the upper bound of $\Tilde{L}(a)$.

Observe that 
\begin{equation*}
        \Tilde{L}(a) = \log \left\vert \frac{\text{det}(\nabla b(a;Z')}{\text{det}(\nabla b(a;Z)} \right\vert \frac{v(b(a;Z);\sigma)}{v(b(a;Z');\sigma)},
\end{equation*}
where $v$ denotes the pdf of Gaussian random variable with mean $0$ and standard deviation $\sigma$.

\noindent
\textbf{Step 1. Bound the ratio of determinants.}
To bound the ratio of determinants, we start by defining \( A = -\nabla b(a; Z') = \sum_{z \in Z'} \nabla^{2}\ell(a; z) + \Delta I \), \( B = -\nabla b(a; Z) = \sum_{z \in Z} \nabla^{2}\ell(a; z) + \Delta I \), and \( C = -\nabla b(a; Z_{-n}) = \sum_{z \in Z_{-n}} \nabla^{2}\ell(a; z) + \Delta I \), so that \( A = C + \nabla^{2}\ell(a; z') \) and \( B = C + \nabla^{2}\ell(a; z) \). Then,
\begin{equation*}
    \begin{split}
        \left\vert \frac{\text{det}(\nabla b(a; Z'))}{\text{det}(\nabla b(a; Z))} \right\vert &= \left\vert \frac{\text{det}(A)}{\text{det}(B)} \right\vert = \left\vert \frac{\text{det}(C + \nabla^{2}\ell(a; z'))}{\text{det}(C + \nabla^{2}\ell(a; z))} \right\vert \\
        &= \left\vert \frac{\text{det}(C)\text{det}(I + C^{-1} \nabla^{2} \ell(a; z'))}{\text{det}(C)\text{det}(I + C^{-1} \nabla^{2}\ell(a; z))} \right\vert \\
        &= \left\vert \frac{\text{det}(I + C^{-1} \nabla^{2}\ell(a; z'))}{\text{det}(I + C^{-1} \nabla^{2}\ell(a; z))} \right\vert \\
        &\leq \left\vert \text{det}(I + C^{-1} \nabla^{2}\ell(a; z')) \right\vert \\
        &\leq \left(1 + \frac{\eta}{\Delta}\right)^{R} = \left(\frac{\Delta + \eta}{\Delta}\right)^{R}.
    \end{split}
\end{equation*}

To justify the first inequality, note that all the eigenvalues of \( C^{-1}\nabla^{2}\ell(a; z) \) are nonnegative because it is positive semidefinite. This is due to \( C^{-1} \) being positive definite and the Hessian matrix \( \nabla^{2}\ell(a; z) \) being positive semidefinite, as it corresponds to a convex function. Thus, all eigenvalues of \( I + C^{-1}\nabla^{2}\ell(a; z) \) are greater than or equal to 1, implying that \( \text{det}(I + C^{-1}\nabla^{2}\ell(a; z)) \geq 1 \). For the second inequality, note that the positive semidefinite matrix \( C^{-1}\nabla^{2}\ell(a; z') \) has rank at most \( R \), which implies that it has at most \( R \) nonzero eigenvalues. Additionally, since \( C = -\nabla b(a; Z_{-n}) = \sum_{z \in Z_{-n}} \nabla^{2}\ell(a; z) + \Delta I \), the eigenvalues of \( C \) are bounded below by \( \Delta \), making the eigenvalues of \( C^{-1} \) at most \( \frac{1}{\Delta} \). Moreover, we assume that the eigenvalues of \( \nabla^2 \ell(a; z') \) are bounded above by \( \eta \). Combining these facts, the eigenvalues of \( C^{-1}\nabla^{2}\ell(a; z') \) are bounded between \( 0 \) and \( \frac{\eta}{\Delta} \), with at most \( R \) nonzero eigenvalues.

\noindent
\textbf{Step 2. Bound the ratio of gaussian pdfs}.
To bound the remaining, we have
\begin{equation*}
\begin{split}
    \log \frac{v(b(a;Z);\sigma)}{v(b(a;Z');\sigma)} &= \log \frac{e^{-\frac{1}{2\sigma^{2}}\Vert b(a;Z)\Vert_{2}^{2}}}{e^{-\frac{1}{2\sigma^{2}}\Vert b(a;Z')\Vert_{2}^{2}}}= \frac{1}{2\sigma^{2}}\left(\Vert b(a;Z')\Vert_{2}^{2} - \Vert b(a;Z)\Vert_{2}^{2}\right).
\end{split}
\end{equation*}

Since $b(a;Z') = b(a;Z) - \nabla \ell(a;Z) + \nabla \ell(a;Z')$, we have 
\begin{equation*}
    \begin{split}
        \Vert b(a;Z')\Vert_{2}^{2} &= \Vert b(a;Z) - \nabla \ell(a;Z) + \nabla \ell(a;Z') \Vert_{2}^{2} \\   
        &\leq \Vert b(a;Z) \Vert_{2}^{2} + \Vert \nabla \ell(a;Z') - \nabla \ell(a;Z) \Vert_{2}^{2} + 2 b(a;Z)^\top [\nabla \ell(a;z;) - \nabla \ell(a;z')]\\
        &\leq \Vert b(a;Z) \Vert_{2}^{2} + 4L^{2} + 2 \Vert b(a;Z) \Vert \Vert \nabla \ell(a;z;) - \nabla \ell(a;z')\Vert\\
        &\leq \Vert b(a;Z) \Vert_{2}^{2} + 4L^{2} + 4 L \Vert b(a;Z) \Vert, \\
        \end{split}
\end{equation*}
where the first and the third inequalities are due to the assumption on the upper bound on the $\ell_{2}-$ norm of gradient and the second inequality is by Cauchy-Schwartz inequality. 

So far we have 
\begin{equation*}
    \begin{split}
        \Tilde{L}(a) &\leq \log \left\vert \frac{\text{det}(\nabla b(a;Z')}{\text{det}(\nabla b(a;Z)} \right\vert + \frac{1}{2\sigma^{2}}\left(4L^{2}+4L\Vert b(a;Z) \vert \right) \\
        &= R\log \left(\frac{\Delta+\eta}{\Delta}\right) + \frac{1}{2\sigma^{2}}\left(4L^{2}+4L\Vert b(a;Z) \Vert \right).
    \end{split}
\end{equation*}

\noindent
\textbf{Step 3. Bound the R\'enyi divergence}.
To establish an upper bound for the R\'enyi divergence, we first analyze the norm of \( b(a; Z) \). Given that \( a \sim \hat{\theta}(Z) \), meaning the estimator is based on the dataset \( Z \), it follows that \( \Vert b(a; Z) \Vert_2 \sim \sigma \chi(d) \). Using Lemma \ref{thm: chi distn bound}, we can then conclude:
\[
\mathbb{E}_{a \sim \hat{\theta}(Z)} \Vert b(a; Z) \Vert_2 \leq \sigma \sqrt{d}.
\]
With this bound in place, we can now proceed to calculate the expectation and bound the R\'enyi divergence.

\begin{equation*}
\begin{split}
    &D_{\alpha}\left(\hat{\theta}(Z) \Vert \hat{\theta}(Z')\right)\\
    &= \frac{1}{\alpha -1}\log \mathbb{E}_{a \sim \hat{\theta}(Z)}\left[e^{(\alpha-1)\Tilde{L}(a)}\right]\\
    &\leq \frac{1}{\alpha -1}\log \mathbb{E}_{a \sim \hat{\theta}(Z)}\left[\exp\left\{(\alpha -1)\left(R\log \left(\frac{\Delta+\eta}{\Delta}\right)+\frac{2L^{2}}{\sigma^{2}}\right)\right\}\exp \frac{(\alpha-1)2L}{\sigma^{2}}\Vert b(a;Z) \Vert_{2}\right]\\
    &= \frac{1}{\alpha-1}\log\left[\exp\left\{(\alpha-1)\left(R\log\left(\frac{\Delta+\eta}{\Delta}\right)+\frac{2L^{2}}{\sigma^{2}}\right)\right\}\mathbb{E}_{a \sim \hat{\theta}(Z)}\exp\left\{\frac{2(\alpha-1)L}{\sigma^{2}}\Vert b(a;Z) \Vert_{2}\right\}\right] \\
    &\leq \frac{1}{\alpha -1}\log \left[\exp \left\{(\alpha-1)\left(R\log\left(\frac{\Delta+\eta}{\Delta}\right)+\frac{2L^{2}}{\sigma^{2}}\right)\right\}\exp\left(\frac{2(\alpha -1)L}{\sigma^{2}}\mathbb{E}_{a \sim \hat{\theta}(Z) }\Vert b(a;Z) \Vert_{2}\right)\right] \\
    &\leq \frac{1}{\alpha -1}\log \left[\exp \left\{(\alpha-1)\left(R\log\left(\frac{\Delta+\eta}{\Delta}\right)+\frac{2L^{2}}{\sigma^{2}}\right)\right\}\exp\left(\frac{2(\alpha -1)L}{\sigma}\sqrt{d}\right)\right]\\
    &= R\log \left(\frac{\Delta+\eta}{\Delta}\right)+\frac{2L^{2}}{\sigma^{2}}+\frac{2L\sqrt{d}}{\sigma}\\
    &\leq \alpha \left(R\log \left(\frac{\Delta+\eta}{\Delta}\right)+\frac{2L^{2}}{\sigma^{2}}+\frac{2L\sqrt{d}}{\sigma}\right),
\end{split}
\end{equation*}
where the second inequality is due to Jensen's inequality and the third inequality is because $\Vert b(a;Z) \Vert_{2}$ is $\sigma \chi (d)$ distributed random variable given the underlying data $Z$ and we observed the upper bound above. 

Finally, some algebra show that $\Delta = \eta \left(\exp\left(\frac{(1-q)\rho}{R}\right)-1\right)^{-1}$ and $\sigma^{2} = \left(\frac{L(\sqrt{d+2q\rho}+\sqrt{d})}{q\rho}\right)^{2}$ result in $R\log \left(\frac{\Delta+\eta}{\Delta}\right)=(1-q)\rho$ and $\frac{2L^{2}}{\sigma^{2}}+\frac{2L\sqrt{d}}{\sigma}=q\rho$ respectively, so that we have the upper bound as $\alpha \rho$, as desired.

\subsection{Proof of Corollary \ref{thm: mle privacy}: Privacy Guarantee of \texttt{PrivateMLE}}
It is sufficient to identity $L, \eta$ and $R$ in our multinomial model.\label{proof: privatemle}

\noindent

\noindent
\textbf{Identifying L: bound on the gradient}.
We begin with rewriting the log-likelihood with denoting the chosen item at time \( t \) as \( i^* \). Therefore, \( y_{ti^*} = 1 \) and \( y_{ti} = 0 \) for \( i \neq i^* \) and $\ell_t(\theta) = -\log \left( \frac{\exp(x_{ti^*}^\top \theta)}{1 + \sum_{j \in S_t} \exp(x_{tj}^\top \theta)} \right)$.

This can be rewritten as:
\[ \ell_t(\theta) = -x_{ti^*}^\top \theta + \log \left( 1 + \sum_{j \in S_t} \exp(x_{tj}^\top \theta) \right).\]

Then the gradient of the negative log-likelihood function at time \( t \) with respect to \( \theta \) is:
\begin{equation*}
    \begin{split}
    \nabla_\theta \ell_t(\theta) &= -x_{ti^*} + \sum_{j \in S_t} \frac{\exp(x_{tj}^\top \theta)}{1 + \sum_{k \in S_t} \exp(x_{tk}^\top \theta)} x_{tj}\\
    &= -x_{ti^*} + \sum_{j \in S_t} p_t(j | S_t, \theta) x_{tj}.
    \end{split}
\end{equation*}

Therefore, the \( \ell_2 \) norm of the gradient can be bounded as:
\begin{equation*}
\|\nabla_\theta \ell_t(\theta)\|_2 \leq \|x_{ti^*}\|_2 + \left\| \sum_{j \in S_t} p_t(j | S_t, \theta) x_{tj} \right\|_2 \leq 1 + 1 = 2.
\end{equation*}
The first inequality is by the triangle inequality. The second inequality is by the Cauchy-Schwarz inequality that gives $\left\| \sum_{j \in S_t} p_t(j | S_t, \theta) x_{tj} \right\|_2 \leq \sum_{j \in S_t} p_t(j | S_t, \theta) \|x_{tj}\|_2 $ and \( \|x_{tj}\|_2 \leq 1 \) for all \( j \in S_t \) and \( \sum_{j \in S_t} p_t(j | S_t, \theta) = 1 \).

\noindent
\textbf{Identifying $\eta$: bound on the eigenvalue of the Hessian} 
The Hessian of the negative log-likelihood function at time \( t \) is given by
\[
\nabla_\theta^2 \ell_t(\theta) = \sum_{j \in S_t} p_t(j \mid S_t, \theta) \left( x_{tj} - \bar{x}_t \right) \left( x_{tj} - \bar{x}_t \right)^\top,
\]
where \( \bar{x}_t = \sum_{j \in S_t} p_t(j \mid S_t, \theta) x_{tj} \).

To bound the largest eigenvalue of this matrix, \( \lambda_{\max}(\nabla_\theta^2 \ell_t(\theta)) = \sup_{\Vert u \Vert_{2}=1} u^\top \nabla_\theta^2 \ell_t(\theta) u \), we begin by considering the quadratic form for any unit vector \( u \in \mathbb{R}^d \):
\[
u^\top \nabla_\theta^2 \ell_t(\theta) u = \sum_{j \in S_t} p_t(j \mid S_t, \theta) \left[ u^\top (x_{tj} - \bar{x}_t) \right]^2.
\]
To bound \( \left[ u^\top (x_{tj} - \bar{x}_t) \right]^2 \), we have $\left| u^\top (x_{tj} - \bar{x}_t) \right| \leq \| u \|_2 \| x_{tj} - \bar{x}_t \|_2 \leq \| x_{tj} - \bar{x}_t \|_2,
$
where the first inequality follows from Cauchy-Schwarz, and the second inequality holds because \( u \) is a unit vector satisfying \( \Vert u \Vert_2 = 1 \).

Given that \( \| x_{tj} \|_2 \leq 1 \) for all \( j \), we can bound \( \| \bar{x}_t \|_2 \) as follows:
\[
\| \bar{x}_t \|_2 = \left\| \sum_{j \in S_t} p_t(j \mid S_t, \theta) x_{tj} \right\|_2 \leq \sum_{j \in S_t} p_t(j \mid S_t, \theta) \| x_{tj} \|_2 \leq \sum_{j \in S_t} p_t(j \mid S_t, \theta) = 1.
\]
Therefore, we can bound \( \| x_{tj} - \bar{x}_t \|_2 \) as
\[
\| x_{tj} - \bar{x}_t \|_2 \leq \| x_{tj} \|_2 + \| \bar{x}_t \|_2 \leq 1 + 1 = 2.
\]

Combining these results gives
\[
u^\top \nabla_\theta^2 \ell_t(\theta) u = \sum_{j \in S_t} p_t(j \mid S_t, \theta) \left[ u^\top (x_{tj} - \bar{x}_t) \right]^2 \leq \sum_{j \in S_t} p_t(j \mid S_t, \theta) \| x_{tj} - \bar{x}_t \|_2^2 \leq \sum_{j \in S_t} p_t(j \mid S_t, \theta) \cdot 4 = 4.
\]
Therefore, the largest eigenvalue of the Hessian is bounded by
\[
\lambda_{\max}\left( \nabla_\theta^2 \ell_t(\theta) \right) \leq 4.
\]

\noindent
\textbf{Identifying $R$: rank of the hessian matrix}
For the hessian matrix at time t,
\begin{equation*}
\nabla_\theta^2 \ell_t(\theta) = \sum_{j \in S_t} p_t(j \mid S_t, \theta) \left( x_{tj} - \bar{x}_t \right) \left( x_{tj} - \bar{x}_t \right)^\top,
\end{equation*}
each of $H_j := p_t(j \mid S_t, \theta) \left( x_{tj} - \bar{x}_t \right) \left( x_{tj} - \bar{x}_t \right)^\top$ is a rank-one, positive semidefinite matrix since it's the outer product of a vector with itself scaled by a non-negative scalar \( p_t(j \mid S_t, \theta) \).

However, we can observe the linear dependency among centered contextual vectors $x_{tj} - \bar{x}_t$. Consider the set of centered feature vectors:
\[
\mathcal{V} = \{ v_j = x_{tj} - \bar{x}_t \mid j \in S_t \}.
\]

We have the following property:
\[
\sum_{j \in S_t} p_t(j \mid S_t, \theta) v_j = \sum_{j \in S_t} p_t(j \mid S_t, \theta) ( x_{tj} - \bar{x}_t ) = \left( \sum_{j \in S_t} p_t(j \mid S_t, \theta) x_{tj} \right) - \bar{x}_t = \bar{x}_t - \bar{x}_t = 0.
\]

It shows that the weighted sum of \( v_j \) with weights \( p_t(j \mid S_t, \theta) \) equals zero, indicating a linear dependency among the vectors \( v_j \). Since there is at least one linear dependency among the \( K \) vectors \( v_j \), the maximum number of linearly independent vectors in \( \mathcal{V} \) is \( K - 1 \).
Therefore, the vectors \( v_j \) lie in a subspace \( \mathcal{W} \subseteq \mathbb{R}^d \) of dimension at most \( K - 1 \).

Recall that the Hessian \( \nabla_\theta^2 \ell_t(\theta) \) is a sum of \( K \) rank-one matrices \( H_j \), whose range spaces are spanned by the vectors \( v_j \) so that contained within the subspace \( \mathcal{W} \). Therefore, the rank of the Hessian, which is the dimension of its range, is at most \( \dim(\mathcal{W}) \leq K - 1 \).

Finally, the rank cannot exceed the dimension \( d \) of the ambient space \( \mathbb{R}^d \) and thus,
\[
R = \text{rank}\left( \nabla_\theta^2 \ell_t(\theta) \right) \leq \min\{ d, K - 1 \}.
\]

\subsection{Proof of Theorem \ref{thm: Cov DP guarantee}: Privacy Guarantee of \texttt{PrivateCov}}

For every node in the binary tree, we inject a noisy matrix by the gaussian mechanism: 
\begin{lemma}\label{thm: zcdp via gauss noise}
(Gaussian mechanism, Lemma 12, \citep{steinke2022composition})
Let $q: \mathcal{X}^{n} \rightarrow \mathbb{R}^{d}$ have sensitivity $\Delta$, that is, $\Vert q(X) - q(X') \Vert_{2} \leq \Delta$ for all $X, X' \in \mathcal{X}^{n}$ such that $X,X'$ differ in a single record. Let $\sigma > 0 .$ Define a randomized algorithm $M: \mathcal{X}^{n} \rightarrow \mathbb{R}^{d}$ by $M(X)=\mathcal{N}(q(X),\sigma^{2}I_{d}).$ Then $M$ is $\rho-$zCDP for $\rho = \frac{\Delta^{2}}{2\sigma^{2}}.$
\end{lemma}

\noindent
\textbf{Step 1. Sensitivity Analysis.} First, we introduce Lemma \ref{thm: sensitivity} that gives the sensitivity with respect to the gram matrix. This is the adapted version of Lemma 3.2 in \cite{biswas2020coinpress} to our MNL problem.

\begin{lemma}\label{thm: sensitivity}
    Let $f(X_{1:n}) = \sum_{t=1}^{n}\sum_{i \in S_{t}}x_{ti}x_{ti}^\top,$ where $\Vert x_{ti} \Vert_{2} \leq 1.$ Then the $\ell_{2}-$sensitivity $= \max_{X_{1:n},X'_{1:n}}\Vert f(X_{1:n}) - f(X'_{1:n}) \Vert_{F},$ where $X_{1:n},X'_{1:n}$ are differ by one user is at most $\sqrt{2K}.$
\end{lemma}
\begin{proof}
Suppose $X_{1:n}$ and $X'_{1:n}$ have a different user at $k.$ For simplicity, by dropping the index $k,$ the $\ell_{2}$-sensitivity is $\Vert \sum_{i \in S}x_{i} 
 x_{i}^\top - \sum_{i \in S'} x_{i}^{\prime} x_{i}^{\prime^\top} \Vert_{F}.$ 
 
 Then we have
 \begin{equation*}
     \begin{split}
        & \Vert \sum_{i \in S} x_{i} 
 x_{i}^\top - \sum_{i \in S'} x_{i}^{\prime} x_{i}^{\prime^\top} \Vert_{F} = \sqrt{\text{tr}\left( \sum_{i \in S} x_{i} 
 x_{i}^\top - \sum_{i \in S'} x_{i}^{\prime} x_{i}^{\prime^\top} \right)^{2}} \\
 &= \sqrt{\text{tr}\left( \left( \sum_{i \in S} x_{i} 
 x_{i}^\top\right)^{2} - \sum_{i \in S} x_{i} 
 x_{i}^\top \sum_{i \in S^{\prime}}x_{i}^{\prime} x_{i}^{\prime^\top} - \sum_{i \in \S^{\prime}}x_{i}^{\prime} x_{i}^{\prime^\top} \sum_{i \in S} x_{i} 
 x_{i}^\top + \left(\sum_{i \in S^{\prime}}x_{i}^{\prime} x_{i}^{\prime^\top}
\right)\right)} \\
& \leq \sqrt{\Vert \sum_{i \in S} x_{i} 
 x_{i}^\top \Vert_{F}^{2} + \Vert \sum_{i \in S^{\prime}}x_{i}^{\prime} x_{i}^{\prime^\top}\Vert_{F}^{2}} \\
 & \leq \sqrt{\sum_{i \in S}\Vert x_{i}x_{i}^\top \Vert_{F}^{2} + \sum_{i \in S^{\prime}} \Vert x_{i}^{\prime}x_{i}^{\prime^\top} \Vert_{F}^{2}}
  = \sqrt{\sum_{i \in S} \left(\Vert x_{i} \Vert_{2}^{4} + \Vert x_{i}^{\prime} \Vert_{2}^{4}\right)} 
  \leq \sqrt{2K},
     \end{split}
 \end{equation*}
where the last inequality is from $\Vert x_{i} \Vert_{2} \leq 1,$ and the cardinality of $S$
 is at most $K.$    
\end{proof}

\noindent
\textbf{Step 2. Apply the Gaussian Mechanism.}

Given the $l_2$-sensitivity with respect to the gram matrix is $\sqrt{2K}$ by Lemma \ref{thm: sensitivity}, suppose each node in the binary tree is $\rho$-zCDP with $\rho = \frac{K}{\sigma^{2}}$. Then from the construction of the tree, each single data, which is a gram matrix in our case, only affects at most $m = \lfloor\log T\rfloor+1$ nodes. Thus, by Lemma \ref{thm: zcdp_composition}, the whole tree is $m\rho$-zCDP. Hence, to satisfy $\rho$-zCDP for the whole tree, we inject noise on each noise with privacy parameter $\rho/m$.

\subsection{Proof of Theorem \ref{thm:jdp guarantee}: Privacy Guarantee of \texttt{DPMNL}}

Finally, we provide the privacy guarantee of our main policy \texttt{DPMNL}. 

\noindent
\textbf{Step 1. Privacy guarantee of perturbed optimistic utilities}

To apply Lemma \ref{thm: billboard lemma}, note that the perturbed optimistic utility estimate on each item for all time horizon $\left\{\{z_{ti}\}_{i \in [N]}\right\}_{t \in [T]}$ is a function of true context vectors of the given user and outputs of two privacy subroutines \texttt{PrivateMLE} and \texttt{PrivateCov} that satisfy $\rho_{1}$-zCDP and $\rho_{2}$-zCDP, respectively. Additionally, the composition of \texttt{PrivateMLE} and \texttt{PrivateCov} satisfies $\rho_{1}+\rho_{2}$-zCDP by Lemma \ref{thm: zcdp_composition}.

Therefore, by Lemma \ref{thm: billboard lemma}, $\left\{\{z_{ti}\}_{i \in [N]}\right\}_{t \in [T]}$ satisfies $(\rho_{1}+\rho_{2})$-joint zCDP.\\

\noindent
\textbf{Step 2. Privacy guarantee of the policy}. Take any $n$-neighboring databases $U$ and $U'$. To explicit which dataset that $\left\{\{z_{ti}\}_{i \in [N]}\right\}_{t \in [T]}$ relies on, we write as $\left\{\{z_{ti}(U)\}_{i \in [N]}\right\}_{t \in [T]}$. Then the JDP guarantee on $\left\{\{z_{ti}\}_{i \in [N]}\right\}_{t \in [T]}$ implies that for all $\alpha >1$,
\[
D_{\alpha}\left( \left\{\{z_{ti}(U)\}_{i \in [N]}\right\}_{t \neq n} \Vert \left\{\{z_{ti}(U')\}_{i \in [N]}\right\}_{t \neq n} \right) \leq (\rho_{1}+\rho_{2})\alpha.
\]
Note then the sequence of assortments prescribed except user $n$ is a post-processing of $\left\{\{z_{ti}(U')\}_{i \in [N]}\right\}_{t \neq n}$ with no additional access to raw data. Therefore, writing $\mathcal{M}_{-t}(U)$ as $g\left(\left\{\{z_{ti}\}_{i \in [N]}\right\}_{t \neq n}\right)$ for some function $g$, by post-processing property Lemma \ref{thm: postprocessing}, we have the desired result:
\[
D_{\alpha}\left( g\left(\left\{\{z_{ti}(U)\}_{i \in [N]}\right\}_{t \neq n}\right) \Vert g\left(\left\{\{z_{ti}(U')\}_{i \in [N]}\right\}_{t \neq n} \right)\right) \leq (\rho_{1}+\rho_{2})\alpha..
\]

\subsection{Statement and Proof of Lemma \ref{thm: pd noisy gram}}

The following lemma guarantees the shifted privatized gram matrix at $t$, $V_{t}$ is positive definite with high probability.

\begin{lemma}\label{thm: pd noisy gram}
    Let $V_{t}= \sum_{n=1}^{t} x_{n i} x_{n i}^\top + N_{t} + 2\lambda I$ be the shifted perturbed gram matrix at time $t$, where $\lambda = \sigma \sqrt{m} \left(2 \sqrt{d} + 2 d^{1/6}\log^{1/3}d + \frac{6(1+(\log d/d)^{1/3})\sqrt{\log d}}{\sqrt{\log (1+(\log d/d)^{1/3}})} + 2 \sqrt{4\log T} \right), m=\lfloor\log_{2}T\rfloor +1$ and $\sigma^{2} = \frac{K}{\rho_{2}}m$.
    Then $V_{t}$ is positive definite with probability at least $1-1/T^{2}$.
\end{lemma}

\begin{proof}
To prove Lemma \ref{thm: pd noisy gram}, it is sufficient to show the all the eigenvalues of $N_{t} + 2\lambda I$ is nonnegative. 

Note that $N_{t}$ is at most $m$ sum of independent $d \times d$ random matrices that follow $N(0, \sigma^{2} I)$, since the depth of binary tree is at most $m$. Moreover, we have the bound on the operator norm of each noisy matrix $N(0, \sigma^{2} I)$ as follows:

\begin{lemma}\label{thm: noisy matrix norm}{(Lemma C.8 in \cite{biswas2020coinpress})}
Let $Y$ be the $d \times d$ matrix where $Y_{ij} \sim N(0, \sigma^{2})$ for $i \leq j,$ and $Y_{ij} = Y_{ji}$ for $i > j.$ Then with probability at least $1-\beta,$ we have the following bound:
\begin{equation*}
    \Vert Y \Vert_{2} \leq \sigma \left(2 \sqrt{d} + 2 d^{1/6}\log^{1/3}d + \frac{6(1+(\log d/d)^{1/3})\sqrt{\log d}}{\sqrt{\log (1+(\log d/d)^{1/3}})} + 2 \sqrt{2\log(1/\beta)} \right).
\end{equation*}
\end{lemma} 

Since $N_{t}$ is the sum of at most $m$ of such gaussian noisy matrices, by Lemma \ref{thm: noisy matrix norm}, we have $\Vert N_{t} \Vert_{2} \leq \lambda,$ with probability at least $1-1/T^{2},$ by replacing $\beta$ with $1/T^{2}$. 

Denoting the eigenvalues of $N_{t}$ by $\lambda_{1},\dots,\lambda_{d},$ the eigenvalues of $N_{t}+2\lambda I$ are $\lambda_{1}+2\lambda, \dots \lambda_{d}+2\lambda.$ By definition of operator norm, 
\begin{equation}\label{equ: noise matrix eval}
\max_{i \in [d]} \vert \lambda_{i} \vert = \Vert N_{t} \Vert_{2} \leq \lambda,    
\end{equation}
so all the eigenvalues of $N_{t} + 2 \lambda I >0,$ as desired. \end{proof}

\section{Proof of Theorem \ref{thm: regret}: Regret Analysis}\label{sec: regret proof}

In this section, we provide a detailed proof of Theorem \ref{thm: regret}. The main challenge lies in bounding the two terms: $\Vert \hat{\theta}_{t} - \theta^{\ast} \Vert_{V_{t}}$ and $\sum_{n=1}^{t}\sum_{i \in S_{n}} \Vert x_{ni} \Vert_{V_{t}^{-1}}$. To address this, we first state the key lemmas and explain their roles in the analysis. We then proceed with the main proof, presenting the detailed regret result. Finally, we provide proofs of the introduced key lemmas.

\subsection{Overview on Key Lemmas}
We begin with introducing Lemma \ref{thm: T0_length}, which helps us to identify the desired length of pure exploration periods $T_{0}$.
\begin{lemma}\label{thm: T0_length}{(Proposition 1 in \cite{oh2021multinomial})} Let $x_{\tau i}$ be drawn i.i.d from some distribution $\nu$ with $\Vert x_{\tau i} \Vert \leq 1$ and $\mathbb{E}[x_{\tau i}x_{\tau i}^\top] \geq \sigma_{0}.$ Define $\Sigma_{T_{0}} = \sum_{\tau =1}^{T_{0}}\sum_{i \in S_{\tau}}x_{\tau i}x_{\tau i}^\top ,$ where $T_{0}$ is the length of random initialization. Suppose we run a random initialization with assortment size $K$ for duration $T_{0}$ which satisfies
\begin{equation*}
    T_{0} \geq \frac{1}{K}\left( \frac{C_{1}\sqrt{d}+C_{2}\sqrt{2\log T}}{\sigma_{0}} \right)^{2} + \frac{2B}{K \sigma_{0}}
\end{equation*}
for some positive, universal constants $C_{1}$ and $C_{2}.$ Then $\lambda_{\text{min}}(V_{T_{0}}) \geq B$ with probability at least $1-1/T^{2}$.
\end{lemma}

Then the following two addresses the estimation error of our perturbed MLE, $\Vert \hat{\theta}_{t} - \theta^{\ast} \Vert_{V_{t}}$.

\begin{lemma}\label{thm: mle bound l2 norm} Let $\Sigma_{T_{0}} = \sum_{n=1}^{T_{0}}\sum_{i \in S_{n}}x_{ni}x_{ni}^\top$, where $T_{0}$ is the length of pure explorations. Let $C_{\rho_{1},T} = \frac{1}{\kappa^{2}}
    \left(\sqrt{\frac{d}{2}\log\left(1+\frac{T}{d}\right)+\log T}
    +4\left(\exp \left(\frac{(1-q)\rho_{1}}{RD_{\text{MLE}}}\right)-1\right)^{-1}
    + \frac{4D_{\text{MLE}}\sqrt{d}\left(\sqrt{d+\frac{2q\rho_{1}}{D_{\text{MLE}}}}+\sqrt{d}\right)}{q\rho_{1}}\sqrt{\frac{\log T}{K}}\right)^{2}.
    $ If $T_{0}$ is large enough such that $\lambda_{\text{min}}\left(\Sigma_{T_{0}}\right) \geq \max\left\{C_{\rho_{1}},K \right\},$
then for any $t \geq T_{0}$, with probability at least $1-1/T^{2},$ we have
\begin{equation*}
    \Vert \hat{\theta}_{t} - \theta^{\ast} \Vert_{2} \leq 1.
\end{equation*}
\end{lemma}

Note that we can figure out the desirable $T_{0}$ that satisfies the condition in Lemma \ref{thm: mle bound l2 norm} by Lemma \ref{thm: T0_length}. Since we need the result in terms of weighted $\ell_{2}$-norm with respect to $V_{t}$, we have the following:

\begin{lemma}\label{thm: mle bound matrix norm} Suppose $\Vert \hat{\theta}_{t} - \theta^{\ast} \Vert_{2} \leq 1$ for all $t \geq T_{0}$. Then for any $t \geq T_{0}$, with probability at least $1-1/t^{2}$,
\begin{equation*}\label{equ: confidence radius}
    \Vert \hat{\theta}_{t} - \theta^{\ast} \Vert_{V_{t}} \leq \alpha_{t},
\end{equation*}
where 
$\alpha_{t} = \frac{1}{\kappa}\left(\sqrt{\frac{d}{2}\log\left(1+\frac{t}{d}\right)+\log t} 
    +4\left(\exp \left(\frac{(1-q)\rho_{1}}{RD}\right)-1\right)^{-1}
    + \frac{4D_{\text{MLE}}\sqrt{d}\left(\sqrt{d+\frac{2q\rho_{1}}{D_{\text{MLE}}}}+\sqrt{d}\right)}{q\rho_{1}}\sqrt{\frac{\log T}{K}}\right)
    +\sqrt{3\lambda}$ and $V_{t}$ is the shifted privatized gram matrix at $t$.
\end{lemma}

 The condition in Lemma \ref{thm: mle bound matrix norm} is satisfied by taking $T_{0}$ that is large enough to satisfy the condition of Lemma \ref{thm: mle bound l2 norm}.
Note that Lemma \ref{thm: mle bound matrix norm} is a finite-sample normality-type estimation error bound for the MLE of the MNL model. Similar analysis result on nonprivate MNL model can be found in \cite{oh2021multinomial}. 
 Additionally, Lemma \ref{thm: mle bound matrix norm} enlighens that we can construct the optimistic utility estimate using the confidence radius $\alpha_{t}$. 

On the other hand, we have the upper bound on $\sum_{n=1}^{t}\sum_{i \in S_{n}} \Vert x_{ni} \Vert_{V_{t}^{-1}}$ as follows: 

\begin{lemma}\label{thm: self normalize}
    If $\lambda_{\text{min}}(\Sigma_{T_{0}}) \geq K$, then we have 
    \begin{equation*}
        \sum_{t=T_{0}+1}^{T} \max_{i \in S_{t}} \Vert x_{ti} \Vert_{V_{t}^{{-1}}}^{2} \leq 2 d \log\left( \frac{d\lambda+TK}{d\lambda+dK}\right),
    \end{equation*}
    with probability at least $1-1/T^{2}$.
\end{lemma}

Next result shows the optimistic expected revenue by our perturbed optimistic utility is indeed an upper bound of the true expected revenue of the oracle optimal assortment.

\begin{lemma}\label{thm: oh lemma 3_5} (Restating Lemma 3,4,5 in \cite{oh2021multinomial})
Suppose $S_{t}^{\ast}$ is the optimal assortment and suppose $S_{t} = \arg\max_{S \subset \mathcal{S}} \Tilde{R}_{t}(S).$ Consider the perturbed utility $z_{ti} = x_{ti}^\top \hat{\theta}_{t-1} + \alpha_{t}\Vert x_{ti} \Vert_{V_{t}^{-1}}.$ If the conditions of Lemma \ref{thm: mle bound l2 norm} hold, then
\begin{equation}\label{equ: oh lemma 3}
    0 \leq z_{ti} - x_{ti}^\top \theta^{\ast} \leq 2 \alpha_{t}\Vert x_{ti} \Vert_{V_{t}^{-1}},
\end{equation}
and we have 
\begin{equation}\label{equ: oh lemma 4}
        R_{t}(S_{t}^{\ast},\theta^{\ast}) \leq \Tilde{R}_{t}(S_{t}^{\ast}) \leq \Tilde{R}_{t}(S_{t}).
    \end{equation}
Finally, we have 
\begin{equation}\label{equ: oh lemma 5}
    \Tilde{R}_{t}(S_{t}) - R_{t}(S_{t},\theta^{\ast}) \leq 2 \alpha_{t} \max_{i \in S_{t}}\Vert x_{ti} \Vert_{V_{t}^{-1}}.
\end{equation}
\end{lemma}
Equation (\ref{equ: oh lemma 3}) implies $z_{ti}$ is an upper bound for the utility $x_{ti}^\top \theta^{\ast}$ if $\theta^{\ast}$ is in the confidence ellipsoid centered at $\hat{\theta}_{t}.$ Equation (\ref{equ: oh lemma 4}) shows that the optimistic expected revenue is an upper bound of the true expected revenue of the oracle optimal assortment. Finally, Equation (\ref{equ: oh lemma 5}) tells the expected revenue has Lipschitz property and bound the immediate regret with the maximum variance over the assortment. Lemma \ref{thm: oh lemma 3_5} is subtly different from original statements in \cite{oh2021multinomial} due to our consideration in privacy. However the proof follows by slight modifying that in \cite{oh2021multinomial}, and hence the proof is omitted.

\subsection{Proof of Theorem \ref{thm: regret}}
The proof can be divided into three steps.

\noindent
\textbf{Step1. Duration of Pure Exploration}. 
Define the event
\begin{equation*}
    \hat{\mathcal{E}} = \{\Vert \hat{\theta}_{t} - \theta^{\ast} \Vert \leq 1, \Vert \hat{\theta}_{t} - \theta^{\ast} \Vert_{V_{t}}\leq \alpha_{t}, 
    \text{ for all } t \geq T_{0} \},
\end{equation*}
which is an event to happen with probability at least $1-1/t^{2}$ if $T_{0}$ is large enough to satisfy 
\[
\lambda_{min}(\Sigma_{T_0}) \geq \max\{C_{\rho_1},K\}.
\]

By Lemma \ref{thm: T0_length}, $T_{0} = \frac{1}{K}\left( \frac{C_{1}\sqrt{d}+C_{2}\sqrt{2\log T}}{\sigma_{0}} \right)^{2} + \frac{2C_{\rho_1}}{K \sigma_{0}}$ for some universal constants $C_{1}$ and $C_{2}$ guarantees $\Vert \hat{\theta}_{t} - \theta^{\ast} \Vert \leq 1$ by Lemma \ref{thm: mle bound l2 norm} and in turn $\Vert \hat{\theta}_{t} - \theta^{\ast} \Vert_{V_{t}} \leq \alpha_{t}$ by Lemma \ref{thm: mle bound matrix norm}.

\noindent
\textbf{Step 2. Decompose the regret}.
Now we decompose the regret into two parts, initial phase and the learning phase:
\begin{equation*}
    \begin{split}
        \mathcal{R}_{T} & = \mathbb{E} \left[ \sum_{t=1}^{T_{0}}(R(S_{t}^{\ast},\theta^{\ast}) - R(S_{t}, \theta^{\ast})) \right] + \mathbb{E} \left[ \sum_{t=T_{0}+1}^{T}(R(S_{t}^{\ast},\theta^{\ast})-R(S_{t},\theta^{\ast}))\right] \\ 
        & \leq T_{0} + \mathbb{E} \left[ \sum_{t=T_{0}+1}^{T}(\Tilde{R}(S_{t},\theta^{\ast})-R(S_{t},\theta^{\ast}))\right],
    \end{split}
\end{equation*}
where the inequality comes from optimistic revenue estimation by Lemma \ref{thm: oh lemma 3_5}. 

We further decompose the regret of learning phase into two components, where $\hat{\mathcal{E}}$ holds and where $\hat{\mathcal{E}}^{c}$ holds:  
\begin{equation*}
    \begin{split}
        \mathcal{R}_{T} &\leq T_{0} + \mathbb{E}\left[ \sum_{t=T_{0}+1}^{T} \left(\Tilde{R}_{t}(S_{t})-R_{t}(S_{t},\theta^{\ast})\right)\mathbbm{1}(\hat{\mathcal{E}})\right] + \mathbb{E}\left[ \sum_{t=T_{0}+1}^{T} \left(\Tilde{R}_{t}(S_{t})-R_{t}(S_{t},\theta^{\ast})\right)\mathbbm{1}(\hat{\mathcal{E}}^{c})\right] \\
        & \leq T_{0} + \mathbb{E}\left[ \sum_{t=T_{0}+1}^{T} \left(\Tilde{R}_{t}(S_{t})-R_{t}(S_{t},\theta^{\ast})\right)\mathbbm{1}(\hat{\mathcal{E}})\right] + \sum_{t=T_{0}+1}^{T}\mathcal{O}(t^{-2}) \\
        & \leq T_{0} +\sum_{t=T_{0}+1}^{T} 2 \alpha_{T} \max_{i \in S_{t}} \Vert x_{ti} \Vert_{V_{t}^{P^{-1}}} + \mathcal{O}(1),
    \end{split}
\end{equation*}
where the last inequality is from Lemma \ref{thm: oh lemma 3_5}.

\noindent
\textbf{Step 3. Combining all the results}.
Applying the Caushy-Schwarz inequality on the second term and applying Lemma \ref{thm: self normalize}, we have 
\begin{equation*}
    \begin{split}
    \sum_{t=T_{0}+1}^{T} 2 \alpha_{T} \max_{i \in S_{t}} \Vert x_{ti} \Vert_{V_{t}^{P^{-1}}} 
        &\leq 2 \alpha_{T} \sqrt{T \sum_{t=T_{0}+1}^{T}\max_{i \in S_{t}}\Vert x_{ti} \Vert_{V_{t}^{-1}}^{2}}
        \leq 2 \alpha_{T} \sqrt{T 2d\log\left(\frac{d\lambda+TK}{d\lambda+dK}\right)}.
    \end{split}
\end{equation*}

Plugging 
\[
\alpha_{T} = \frac{1}{\kappa}\left(\sqrt{\frac{d}{2}\log\left(1+\frac{T}{d}\right) + \log T}
    + 4\left(\exp \left(\frac{(1-q)\rho_{1}}{RD_{\text{MLE}}}\right)-1\right)^{-1}\right.
\]
\[
    \left.+ \frac{4D_{\text{MLE}}\sqrt{d}\left(\sqrt{d+\frac{2q\rho_{1}}{D_{\text{MLE}}}} + \sqrt{d}\right)}{q\rho_{1}}\sqrt{\frac{\log T}{K}}\right)
    + \sqrt{3\lambda},
\]
we have 
\begin{equation}\label{equ: regret}
\begin{aligned}
    \mathcal{R}_{T} &\leq T_{0} + \frac{1}{\kappa}\sqrt{4d^{2}T\log\left(1+\frac{T}{d}\right)\log\left(\frac{d\lambda+TK}{d\lambda+dK}\right) + 8dT\log T \log\left(\frac{d\lambda+TK}{d\lambda+dK}\right)}\\
    &\quad + \frac{4}{\kappa}\left(\exp \left(\frac{(1-q)\rho_{1}}{RD_{\text{MLE}}}\right)-1\right)^{-1}\sqrt{4dT\log\left(\frac{d\lambda+TK}{d\lambda+dK}\right)}\\
    &\quad + \frac{8\sqrt{2}D_{\text{MLE}}d\left(\sqrt{d+\frac{2q\rho_{1}}{D_{\text{MLE}}}}+\sqrt{d}\right)}{\kappa \sqrt{K} q \rho_{1}}\sqrt{T \log T \log\left(\frac{d\lambda+TK}{d\lambda+dK}\right)}\\
    &\quad + \sqrt{24 d T \lambda \log\left(\frac{d\lambda+TK}{d\lambda+dK}\right)} + \mathcal{O}(1).
\end{aligned}
\end{equation}

Finally, since $\kappa$ and $q$ are constants, treating $K$ as a constant, taking $D_{\text{MLE}}=O(dK\log T)$ and plugging the expression of $\lambda$, we can simplify the upper bound as
\begin{equation*}
    R_{T} \leq \Tilde{O}\left(\left(d+\frac{d^{5/2}}{\rho_{1}}+\frac{d^{3/4}}{\rho_{2}^{1/4}}\right)\sqrt{T}\right).
\end{equation*}

\subsection{Proofs of the Key Lemmas}

\subsubsection{Proof of Lemma \ref{thm: mle bound l2 norm}}

\textbf{Step1. Define the key tool $G_{t}(\theta)$}. Define
\begin{equation*}
    G_{t}(\theta) := \sum_{n=1}^{t}\sum_{i \in S_{n}}\left(p_{ni}(\theta)-p_{ni}(\theta^{\ast})\right)x_{ni} + \Delta(\theta - \theta^{\ast}).
\end{equation*}

Then since the private MLE $\hat{\theta}_{t}$ is the solution of the following:
\begin{equation*}
\sum_{n=1}^{t}\sum_{i \in S_{n}} \left( p_{ni}(\theta)-y_{ni}\right)x_{ni} + \Delta \theta + b = 0, 
\end{equation*}
we have 
\begin{equation}\label{equ: g_t}
    \begin{split}
        G_{t}(\hat{\theta}_{t}) &= \sum_{n=1}^{t}\sum_{i \in S_{n}} \left( p_{ni}(\hat{\theta}_{t})-p_{ni}(\theta^{\ast})) \right) x_{ni} + \Delta (\hat{\theta}_{t}-\theta^{\ast})\\
    &= \sum_{n=1}^{t}\sum_{i \in S_{n}} \left( p_{ni}(\hat{\theta})-y_{ni}\right)x_{ni} + \sum_{n=1}^{t}\sum_{i \in S_{n}}\left(y_{ni} - p_{ni}(\theta^{\ast})\right)x_{ni}+ \Delta (\hat{\theta}_{t}-\theta^{\ast})\\
    &= \sum_{n=1}^{t} \sum_{i \in S_{n}} \epsilon_{ni} x_{ni}-\Delta \theta^{\ast} - b\\
    &= Z_{t}-\Delta \theta^{\ast} - b,
    \end{split}
\end{equation}
where $\epsilon_{ni}$ is a subGaussian random variable, namely, $\epsilon_{ni} \sim \text{subG}(\sigma^{2})$ for all $n$ with $\sigma^{2}=1/4$ in our multinomial model and $Z_{t} :=\sum_{n=1}^{t} \sum_{i \in S_{n}} \epsilon_{ni} x_{ni}$. \\

\noindent
\textbf{Step 2. Relate $\Vert \hat{\theta}_{t}-\theta^{\ast} \Vert_{2}$ with $G_{t}(\hat{\theta}_{t})$ by Lemma \ref{thm: chen_mle_conv}}. 
Aiming to apply Lemma \ref{thm: chen_mle_conv}, let's first show that $G_{t}(\theta)$ is an injection from $\mathbb{R}^{d}$ to $\mathbb{R}^{d}$ by verifying 
\begin{equation*}
    \left(\theta_{1} - \theta_{2}\right)^\top \left(J_{t}(\theta_{1}) - J_{t}(\theta_{2}) \right) > 0, \text{ for any } \theta_{1} \neq \theta_{2} \in \mathbb{R}^{d}. 
\end{equation*}

For any $\theta_{1},\theta_{2} \in \mathbb{R}^{d},$ the mean value theorem implies there exists $\Bar{\theta}$ that lies on the line segment between $\theta_{1}$ and $\theta_{2}$ such that
\begin{equation*}
    \begin{split}
        G_{t}({\theta_{1}}) - G_{t}(\theta_{2}) &= \left[\sum_{n=1}^{t} \sum_{i \in S_{n}} \sum_{j \in S_{n}} \nabla_{j}p_{ni}(\Bar{\theta})x_{ni}x_{nj}^\top + \Delta I_{d} \right]\left( \theta_{1} - \theta_{2} \right) \\
        &= \sum_{n=1}^{t} \left[\sum_{i \in S_{n}}p_{ni}(\Bar{\theta})x_{ni} x_{ni}^\top - \sum_{i \in S_{n}} \sum_{j \in S_{n}} p_{ni}(\Bar{\theta})p_{nj}(\Bar{\theta})x_{nj}x_{nj}^\top + \Delta I_{d} \right]\left(\theta_{1}-\theta_{2}\right) \\
        &:= \sum_{n=1}^{t} H_{n}(\theta_{1}-\theta_{2}) + \Delta\left(\theta_{1}-\theta_{2}\right) ,\\
    \end{split}
\end{equation*}
where $H_{n} = \sum_{i \in S_{n}}p_{ni}(\Bar{\theta})x_{ni}x_{ni}^\top - \sum_{i \in S_{n}} \sum_{j \in S_{n}}p_{ni}(\Bar{\theta})p_{nj}(\Bar{\theta})x_{ni}x_{nj}^\top$ is a hessian of a negative log-likelihood. Moreover, $H_{n}$ can be lower bounded as follows:
\begin{equation}\label{equ: hessian lower bound}
    \begin{split}
        H_{n} &= \sum_{i \in S_{n}}p_{ni}(\Bar{\theta})x_{ni}x_{ni}^\top - \frac{1}{2}\sum_{i \in S_{n}} \sum_{j \in S_{n}} p_{ni}(\Bar{\theta}) p_{nj}(\Bar{\theta}) \left(x_{ni}x_{nj}+x_{nj}x_{ni}^\top\right) \\
        &\succcurlyeq \sum_{i \in S_{n}} p_{ni}(\Bar{\theta})x_{ni}x_{ni}^\top - \frac{1}{2}\sum_{i \in S_{n}} \sum_{j \in S_{n}}p_{ni}(\Bar{\theta})p_{nj}(\Bar{\theta}\left(x_{ni}x_{ni} + x_{nj}x_{nj}^\top\right) \\
        &= \sum_{i \in S_{n}}p_{ni}(\Bar{\theta})x_{ni}x_{ni}^\top - \sum_{i \in S_{n}} \sum_{j \in S_{n}} p_{ni}(\Bar{\theta}) p_{nj}(\Bar{\theta})x_{ni}x_{ni}^\top \\
        &= \sum_{i \in S_{n}} \left(1-\sum_{j \in S_{n}}p_{nj}(\Bar{\theta})x_{ni}x_{ni}^\top \right) = \sum_{i \in S_{n}} p_{ni}(\Bar{\theta})p_{n0}(\Bar{\theta})x_{ni}x_{ni}^\top 
        \end{split}
\end{equation}
where $p_{n0}(\Bar{\theta})$ is the probability of choosing the no purchase option under parameter $\Bar{\theta}$. Note that the first inequality is by $x_{i}x_{i}^\top + x_{j}x_{j}^\top \succcurlyeq x_{i}x_{j}^\top + x_{j}x_{i}^\top$. Therefore,
\begin{equation*}
    \begin{split}
        \left(\theta_{1}-\theta_{2}\right)(G_{t}(\theta_{1}-G_{t}(\theta_{2})) &\geq \left(\theta_{1}-\theta_{2}\right)\left[\sum_{i \in S_{n}} p_{ni}(\Bar{\theta})p_{n0}(\Bar{\theta})x_{ni}x_{ni}^\top \right]\left(\theta_{1}-\theta_{2}\right) + \Delta \left(\theta_{1}-\theta_{2}\right)^\top \left(\theta_{1}-\theta_{2}\right)\\
        &\geq \Delta \left(\theta_{1}-\theta_{2}\right)^\top \left(\theta_{1}-\theta_{2}\right) > 0,
    \end{split}
\end{equation*}
as desired.

Next, consider $\mathcal{B}_{1}(\theta^{\ast}) = \left\{\theta \in \mathbb{R}^{d}: \Vert \theta - \theta^{\ast} \Vert_{2} \leq 1 \right\}$ and $\mathcal{S}_{1}(\theta^{\ast}) = \left\{\theta \in \mathbb{R}^{d}: \Vert \theta - \theta^{\ast} \Vert_{2} = 1 \right\}$. Observe first that $\mathcal{B}_{1}(\theta^{\ast})$ is convex. Therefore, if $\theta_{1},\theta_{2} \in \mathcal{B}_{1}(\theta^{\ast})$, then its line segment, or equivalently the convex combination $\Bar{\theta}$ of $\theta_{1},\theta_{2}$ should also be contained in $\mathcal{B}_{1}(\theta^{\ast})$. Therefore, by Assumption \ref{ass: kappa}, such $\Bar{\theta}$ satisfies 
\begin{equation}\label{equ: kappa lower bound}
    \sum_{n=1}^{t}\sum_{i \in S_{n}} p_{ni}(\Bar{\theta})p_{n0}(\Bar{\theta})x_{ni}x_{ni}^\top \geq \kappa \sum_{n=1}^{t}\sum_{i \in S_{n}}x_{ni}x_{ni}^\top = \kappa \Sigma_{t}.
\end{equation}

In addition, it is clear that $\theta^{\ast} \in \mathcal{B}_{1}(\theta^{\ast})$ and $G_{t}(\theta^{\ast}) = 0$ by its construction.

Observe then for any $\theta \in \mathcal{B}_{1}(\theta^{\ast})$,
\begin{equation}\label{equ: Sigma norm bound}
    \begin{split}
        \Vert G_{t}(\theta) \Vert_{\Sigma_{t}^{-1}}^{2} &= \Vert G_{t}(\theta) - G_{t}(\theta^{\ast}) \Vert_{\Sigma_{t}^{-1}}^{2} \\
        &\geq (\theta - \theta^{\ast})^\top (\kappa \Sigma_{t} + \Delta I_{d})\Sigma_{t}^{-1}(\kappa \Sigma_{t} + \Delta I_{d})(\theta - \theta^{\ast}) \\
        &=\kappa^{2}(\theta - \theta^{\ast})^\top \Sigma_{t}(\theta - \theta^{\ast}) + 2 \kappa \Delta (\theta - \theta^{\ast})^\top (\theta - \theta^{\ast}) + \Delta^{2}(\theta - \theta^{\ast})^\top \Sigma_{t}^{-1}(\theta - \theta^{\ast})\\
        &= \kappa^{2}\Vert \theta - \theta^{\ast}\Vert_{\Sigma_{t}}^{2} + 2 \kappa \Delta\Vert \theta - \theta^{\ast}\Vert_{2}^{2}\\
        &\geq \kappa^{2}\lambda_{\text{min}}(\Sigma_{t})\Vert \theta - \theta^{\ast}\Vert_{2}^{2} + 2 \kappa \Delta\Vert \theta - \theta^{\ast}\Vert_{2}^{2},
    \end{split}
\end{equation}
where the first inequality is by equations (\ref{equ: hessian lower bound}) and (\ref{equ: kappa lower bound}) and the second inequality holds because $\Sigma_{t}^{-1}$ is positive definite and $\Delta >0$.


In particular, for $\theta \in \mathcal{S}_{1}(\theta^{\ast})$, we have 
\begin{equation*}
    \Vert G_{t}(\theta) \Vert_{\Sigma_{t}^{-1}}^{2} \geq \kappa^{2}\lambda_{\text{min}}(\Sigma_{t}) + 2 \kappa \Delta.
\end{equation*}

Therefore, by taking $r = \kappa^{2}\lambda_{\text{min}}(\Sigma_{t}) + 2 \kappa \Delta$ in Lemma \ref{thm: chen_mle_conv} with an injective function $G_{t}(\theta)$, we have 
\begin{equation*}
    \left\{\theta \in \mathbb{R}^{d}: \Vert G_{t}(\theta) \Vert_{\Sigma_{t}^{-1}}^{2} \leq \kappa^{2}\lambda_{\text{min}}(\Sigma_{t}) + 2 \kappa \Delta \right\} \subseteq \left\{\theta \in \mathbb{R}^{d}: \Vert \theta - \theta^{\ast} \Vert \leq 1 \right\}.
\end{equation*}

Since $V_{t} \succcurlyeq V_{T_{0}}$ for all $t \geq T_{0}$ and $2 \kappa \Delta >0$, we have
\begin{equation*}
    \left\{\theta \in \mathbb{R}^{d}: \Vert G_{t}(\theta) \Vert_{\Sigma_{t}^{-1}} \leq \kappa \sqrt{\lambda_{\text{min}}(\Sigma_{T_{0}})}\right\} \subseteq \left\{\theta \in \mathbb{R}^{d}: \Vert \theta - \theta^{\ast} \Vert \leq 1 \right\}.
\end{equation*}

Therefore, if $\hat{\theta}_{t}$ satisfies $\Vert G_{t}(\hat{\theta}_{t}) \Vert_{\Sigma_{t}^{-1}} \leq \kappa \sqrt{\lambda_{\text{min}}(\Sigma_{T_{0}})}$, we have $\Vert \hat{\theta}_{t} - \theta^{\ast} \Vert \leq 1$.\\

\noindent
\textbf{Step 3. Upper bound on $\Vert G_{t}(\hat{\theta})_{t} \Vert_{\Sigma_{t}^{-1}}$}

Applying the triangular inequality on Equation (\ref{equ: g_t}), we have
\begin{equation*}
    \begin{split}
        \Vert G_{t}(\hat{\theta})_{t} \Vert_{\Sigma_{t}^{-1}} &\leq \Vert Z_{t} - \Delta \theta^{\ast} - b \Vert_{\Sigma_{t}^{-1}}\\
        &\leq \underbrace{\Vert Z_{t}\Vert_{\Sigma_{t}^{-1}}}_{(a)} + \underbrace{\Vert \Delta\theta^{\ast}\Vert_{\Sigma_{t}^{-1}}}_{(b)} +\underbrace{\Vert b\Vert_{\Sigma_{t}^{-1}}}_{(c)}  
    \end{split}
\end{equation*}

To bound $(a)$, we utilize Theorem 1 of \cite{abbasi2011improved}, which states that if the noise $\epsilon_{ni}$ is sub-Gaussian with parameter $\sigma$, 
\begin{equation*}
    \Vert Z_{t} \Vert_{\Sigma_{t}^{-1}}^{2} \leq 2 \sigma^{2} \log\left(\frac{\det (\Sigma_{t})^{1/2}\det (\Sigma_{T_{0}})^{-1/2}}{\delta}\right),
\end{equation*}
with probability at least $1-\delta$. Here $\sigma = 1/2$ in our case. 

Then by Lemma \ref{thm: oh_lemma 8}, 
\begin{equation*}
    \begin{split}
        \Vert Z_{t} \Vert_{\Sigma_{t}^{-1}}^{2} \leq 2 \sigma^{2} \left[\frac{d}{2}\log \left( \frac{\text{trace}(\Sigma_{T_{0}})+tK}{d}\right) -\frac{1}{2}\log\det(\Sigma_{T_{0}}) + \log \frac{1}{\delta} \right].
    \end{split}
\end{equation*}
Denoting $\lambda_{1}, \cdots, \lambda_{d}$ the eigenvalues of $V_{T_{0}}$ and $\Bar{\lambda} = \frac{\sum_{i}^{d} \lambda_{i}}{d},$ we have
\begin{equation*}
    \begin{split}
        \Vert Z_{t} \Vert_{\Sigma_{t}^{-1}}^{2} &\leq 2 \sigma^{2} \left[\frac{2}{d} \log \left(\Bar{\lambda}+\frac{tK}{d}\right) - \frac{d}{2}\log \Bar{\lambda} + \frac{d}{2}\log \Bar{\lambda} - \frac{1}{2}\log\det(\Sigma_{T_{0}}) + \log \frac{1}{\delta}\right] \\
        &= 2 \sigma^{2} \left[\frac{d}{2}\log \left(1+\frac{tK}{d \Bar{\lambda}}\right) + \frac{1}{2}\sum_{i}^{d} \log \frac{\Bar{\lambda}}{\lambda_{i}} + \log \frac{1}{\delta} \right] \\
        &\leq 2 \sigma^{2} \left[\frac{d}{2}\log \left(1+\frac{tK}{d \lambda_{\text{min}}}\right)  + \frac{d}{2}\log \frac{\Bar{\lambda}}{\lambda_{\text{min}}} + \log \frac{1}{\delta}\right] \\
        &\leq 2 \sigma^{2} \left[\frac{d}{2}\log\left(1+\frac{t}{d}\right)+\frac{d}{2}\log\frac{\Bar{\lambda}}{K}+\log\frac{1}{\delta}\right]\\
        &\leq 2 \sigma^{2} \left[d\log\left(1+\frac{t}{d}\right)+\log\frac{1}{\delta}\right]
    \end{split}
\end{equation*}
where the third inequality is by $\lambda_{\text{min}}(\Sigma_{T_{0}}) \geq K$ and the last inequality is from $d\Bar{\lambda} = \text{trace}(\Sigma_{T_{0}}) \leq tK$. Then using $\sigma^{2} = \frac{1}{4}$ and putting $\delta = \frac{1}{T^{2}}$, we have the desired upper bound on $(a)$.

To upper bound $(b)$, observe that 
\begin{equation*}
    \Delta \Vert \theta^{\ast}\Vert_{\Sigma_{t}^{-1}} \leq \Delta \frac{\Vert \theta^{\ast}\Vert_{2}}{\sqrt{\lambda_{\text{min}}(\Sigma_{t})}} \leq \frac{\Delta}{\sqrt{\lambda_{\text{min}}(\Sigma_{T_{0}})}} \leq \frac{\Delta}{\sqrt{K}}=4\left(\exp\left(\frac{(1-q)\rho_{1}}{RD}\right)-1\right)^{-1},
\end{equation*}
where the second inequality is due to the assumption that $\Vert \theta^{\ast} \Vert_{2} \leq 1$ and the third inequality is due to the assumption that we take $T_{0}$ such that $\lambda_{\text{min}}(\Sigma_{T_{0}}) \geq K$.

To upper bound $(c)$, we use Lemma \ref{thm: kifer l2 bound} with $\sigma = \frac{2D_{\text{MLE}}\left(\sqrt{d+\frac{2q\rho_{1}}{D_{\text{MLE}}}}+\sqrt{d}\right)}{q\rho_{1}}$ and $\gamma = 1/T^{2}$, as such $\sigma$ is what we use for a single implement of \texttt{PrivateMLE}. Then with probability at least $1-1/T^{2}$, we have 
\begin{equation*}
    \begin{split}
        \Vert b_{t} \Vert_{\Sigma_{t}^{-1}} \leq \frac{\Vert b_{t} \Vert_{2}}{\sqrt{\lambda_{\text{min}}(\Sigma_{t})}}
        \leq \frac{\Vert b_{t} \Vert_{2}}{\sqrt{K}}
        \leq \frac{4D_{\text{MLE}}\sqrt{d}\left(\sqrt{d+\frac{2q\rho_{1}}{D_{\text{MLE}}}}+\sqrt{d}\right)}{q\rho_{1}}\sqrt{\frac{\log T}{K}} .
    \end{split}
\end{equation*}
In all, $\Vert G_{t}(\hat{\theta}_{t})\Vert_{\Sigma_{t}^{-1}}$ is upper bounded by:
\begin{equation}\label{equ: mle estimation error}
\left(\sqrt{\frac{d}{2}\log\left(1+\frac{t}{d}\right) + \log t} 
       + 4\left(\exp \left(\frac{(1-q)\rho_{1}}{RD_{\text{MLE}}}\right) - 1\right)^{-1} + \frac{4D_{\text{MLE}}\sqrt{d\log T}\left(\sqrt{d + \frac{2q\rho_{1}}{D_{\text{MLE}}}} + \sqrt{d}\right)}{q\rho_{1}\sqrt{K}}\right)^{2}.    
\end{equation}

Therefore, taking $T_{0}$ be sufficiently large that $\lambda_{\text{min}}(\Sigma_{T_{0}}) \geq \max\{C_{\rho_{1},T},K\}$ leads to $\hat{\theta}_{t} \in \mathcal{B}_{1}(\theta^{\ast})$, as desired.

\subsection{Proof of Lemma \ref{thm: mle bound matrix norm}}
Given that $\Vert \hat{\theta}_{t} - \theta^{\ast} \Vert \leq 1$ with probability at least $1-1/T^{2}$, we show the upper bound on $\Vert \hat{\theta}_{t}-\theta^{\ast} \Vert_{V_{t}}$. 

\noindent
\textbf{Step 1. Upper bound on $V_{t}$}. Note that taking the noisy random matrix $V_{t}$ in the weighted $\ell_{2}$-norm is technically challenging. To address this, we bound $V_{t} = \Sigma_{t} + N_{t} + 2\lambda I$ by a deterministic matrix. Note that Equation (\ref{equ: noise matrix eval}) implies that $N_{t} + \lambda I$ is positive semi-definite with high probability, and so
\begin{equation}\label{equ: V upper bound}
    V_{t} = \Sigma_{t} + N_{t} + 2\lambda I \preccurlyeq \Sigma_{t} + 3\lambda I.
\end{equation}

\noindent
\textbf{Step 2. Upper bound the estimation error}. 
\begin{equation*}
    \begin{aligned}
       &\kappa^{2} \Vert \hat{\theta}_{t}-\theta^{\ast} \Vert_{V_{t}}^{2} 
       = \kappa^{2}(\hat{\theta}_{t}-\theta^{\ast})^\top V_{t} (\hat{\theta}_{t}-\theta^{\ast}) 
       \leq \kappa^{2}(\hat{\theta}_{t}-\theta^{\ast})^\top \left(\Sigma_{t} + 3\lambda I\right)(\hat{\theta}_{t}-\theta^{\ast})\\
       &\leq \kappa^{2} \Vert \hat{\theta}_{t} - \theta^{\ast} \Vert_{\Sigma_{t}}^{2} + \kappa^{2} 3 \lambda 
       \leq \Vert G_{t}(\hat{\theta}_{t}) \Vert_{\Sigma_{t}^{-1}}^{2} + \kappa^{2}3\lambda\\ 
       &\leq \left(\sqrt{\frac{d}{2}\log\left(1+\frac{t}{d}\right) + \log t} 
       + 4\left(\exp \left(\frac{(1-q)\rho_{1}}{RD_{\text{MLE}}}\right) - 1\right)^{-1}\right. \\
       &\quad \left. + \frac{4D_{\text{MLE}}\sqrt{d}\left(\sqrt{d + \frac{2q\rho_{1}}{D_{\text{MLE}}}} + \sqrt{d}\right)}{q\rho_{1}}\sqrt{\frac{\log T}{K}}\right)^{2} 
       + \kappa^{2}3 \lambda,
    \end{aligned}
\end{equation*}
where the first inequality is due to Equation (\ref{equ: V upper bound}) and the second inequality is by assumption that $\Vert \hat{\theta}_{t}-\theta^{\ast} \Vert_{2} \leq 1$. The third inequality stems from Equation (\ref{equ: Sigma norm bound}) that implies $ \kappa^{2} \Vert \hat{\theta}_{t} - \theta^{\ast} \Vert_{\Sigma_{t}}^{2} \leq \Vert G_{t}(\hat{\theta}_{t}) \Vert_{\Sigma_{t}^{-1}}^{2}$. Finally the fourth inequality is by Equation (\ref{equ: mle estimation error}) that gives an upper bound of $\Vert G_{t}(\hat{\theta}_{t}) \Vert_{\Sigma_{t}^{-1}}^{2}$.

Dividing both terms by $\kappa^{2}$ and taking squared root on both sides, and finally leveraging the observation that $\sqrt{x+y}\leq \sqrt{x}+\sqrt{y}$ for positive $x,y$, we have the desired result.

\subsection{Proof of Lemma \ref{thm: self normalize}}
Suppose $T_{0}$ is large enough that $\lambda_{min}\left(\Sigma_{T_0}\right) \geq K$. We first show that for all $t \geq T_{0}$, with probability at least $1-1/T^{2}$, 
\begin{equation*}
    \left(V_{t}\right)^{-1} \preccurlyeq \left(\Sigma_{t} + \lambda I \right)^{-1},
\end{equation*}
where $\Sigma_{t} = \sum_{n=1}^{t}\sum_{i \in S_{n}}x_{ni}x_{ni}^\top$.

This is because Equation (\ref{equ: noise matrix eval}) which implies $N_{t}+\lambda I$ is positive semi-definite with high probability and $V_{t} - (\Sigma_{t}+\lambda I) = N_{t} + \lambda I$.

Therefore, 
\begin{equation*}
    \begin{split}
        \sum_{n=T_{0}+1}^{t} \sum_{i \in S_{n}} \Vert x_{ni} \Vert_{\left(V_{n-1}\right)^{-1}}^{2} &\leq \sum_{n=T_{0}+1}^{t} \sum_{i \in S_{n}} \Vert x_{ni} \Vert_{\left(\Sigma_{n-1}+\lambda I\right)^{-1}}^{2}\\
        &\leq 2 \sum_{n=T_{0}+1}^{t} \sum_{i \in S_{n}} \log \left(1+\sum_{i \in S_{n}} \Vert x_{ni} \Vert_{\left(\Sigma_{n-1}+\lambda I \right)^{-1}}^{2} \right) \\
        &= 2 \log \prod_{n=T_{0}+1}^{t}\left( 1 + \sum_{i \in S_{n}}\Vert x_{ni} \Vert_{\left(\Sigma_{n-1}+\lambda I\right)^{-1}}^{2}\right),
    \end{split}
\end{equation*}
where the second inequality is by $x \leq \log(1+x)$ for $x \in [0,1)$ and $\sum_{i \in S_{n}} \Vert x_{ni} \Vert_{\left(\Sigma_{n-1}+\lambda I\right)^{-1}}^{2} \leq \frac{\sum_{i \in S_{n}}\Vert x_{ni} \Vert_{2}^{2}}{K+\lambda} \leq \frac{K}{K+\lambda} < 1$.

To bound further, let $\lambda_{1},\cdots,\lambda_{d}$ be the eigenvalues of $\sum_{i \in S_{n}}x_{ni}x_{ni}^\top$, which are all nonnegative. Then
\begin{equation*}
\begin{split}
    \det \left(I + \sum_{i \in S_{n}} x_{ni}x_{ni}^\top \right) &= \prod_{j=1}^{d}\left(1+\lambda_{j}\right) \\
    &\geq 1+\sum_{j}^{d}\lambda_{j} = 1+\text{trace}\left(\sum_{i \in S_{n}}x_{ni}x_{ni}^\top \right) 
    = 1 + \sum_{i \in S_{n}} \Vert x_{ni} \Vert_{2}^{2}.
\end{split}
\end{equation*}

Therefore, by applying this repeatedly, we have
\begin{equation*}
    \begin{split}
        \det \left(V_{t-1}+\lambda I\right) &= \det \left(V_{t-2}+\lambda I + \sum_{i \in S_{t-1}}x_{t-1 i}x_{t-1 i}^\top \right) \\
        &= \det \left(V_{t-2}+\lambda I\right)\det \left(I+\sum_{i \in S_{t-1}} \left(V_{t-2}+\lambda I\right)^{1/2}x_{t-1 i}x_{t-1 i}^\top\left(V_{t-2}+\lambda I\right)^{1/2}
        \right) \\
        &\geq \det \left(V_{t-2}+\lambda I\right)\left(1+\sum_{i \in S_{t-1}}\Vert x_{t-1i}\Vert_{\left(V_{t-2}+\lambda I\right)^{-1}}^{2}\right)\\
        &\geq \det\left(V_{T_{0}}+\lambda I\right)\prod_{n=T_{0}+1}^{t}\left(1+\sum_{i \in S_{n}} \Vert x_{ni} \Vert_{\left(V_{n-1}+\lambda I\right)^{-1}}^{2}\right).
    \end{split}
\end{equation*}

We next bound $\frac{\det(V_{t-1}+\lambda I)}{\det(V_{T_{0}}+\lambda I)}$. Combining \ref{thm: oh_lemma 8} and $\lambda_{\text{min}}(V_{T_{0}}) \geq K$, we have
\begin{equation*}
    \begin{split}
        \frac{\det(V_{t-1}+\lambda I)}{\det(V_{T_{0}}+\lambda I)} &\leq \frac{1}{(K+\lambda)^{d}} \times \left( \frac{d\lambda + tK}{d}\right)^{d}\\
        &= \left( \frac{d\lambda+tK}{d\lambda+dK}\right)^{d},
    \end{split}
\end{equation*}
and so we have the desired result.

\section{Implementation Details}\label{sec: simulation details}

In this section, we provide the details on numerical studies. We first formally introduce the privacy definitions $(\epsilon,\delta)$-DP and $(\epsilon,\delta)$-JDP in Definition \ref{def: epsdel DP} and \ref{def: epsdel JDP} respectively. Then we give the formal result on objective perturbation for bounded $(\epsilon,\delta)$-DP and investigate the implement details on \texttt{PrivateCov} for $(\epsilon,\delta)$-DP as well.

\subsection{Preliminary: Definitions}

\begin{definition}\label{def: epsdel DP} ($(\epsilon,\delta)$-Differential Privacy \citep{dwork2006differential})
    A bandit mechanism $M: \mathcal{U}^{T} \rightarrow \mathcal{S}^{T}$ satisfies $(\epsilon, \delta)$-DP if for any $t$-neighboring datasets $U$ and $U^\prime$ for all $t \in [T]$ and any outcomes $S$,
\begin{equation*}
    \mathbb{P}(M(U) \in S) \leq e^{\epsilon}\mathbb{P}(M(U^{\prime}) \in S) + \delta.
\end{equation*}
\end{definition}

Definition \ref{def: epsdel DP} ensures that the probabilities of output sequences from \(t\)-neighboring datasets remain similar, thereby protecting user \(t\) from being uniquely identified. The privacy level is governed by the parameters \(\epsilon\) and \(\delta\), with smaller values indicating stronger privacy protection. 

Compared to $\rho$-zCDP that relies on R\'enyi divergence, Definition \ref{def: epsdel DP} measures the similarity between distributions of two outputs by their likelihood ratio. This difference remains the same for the presentation of JDP.

\begin{definition}\label{def: epsdel JDP} ($(\epsilon,\delta)$-Joint Differential Privacy \citep{hsu2014private,shariff2018differentially})
    A bandit mechanism $M: \mathcal{U}^{T} \rightarrow \mathcal{S}^{T}$ satisfies $(\epsilon, \delta)$-DP if for any $t$-neighboring datasets $U$ and $U^\prime$ for all $t \in [T]$ and any outcomes $S$,
\begin{equation*}
    \mathbb{P}(M_{-t}(U) \in S) \leq e^{\epsilon}\mathbb{P}(M_{-t}(U^{\prime}) \in S) + \delta.
\end{equation*} 
\end{definition}

The conversion from $\rho$-zCDP guarantee to $(\epsilon,\delta)$-DP guarantee is as follows: 
\begin{lemma}\label{thm: zcdp conversion}(Remark 15 in \citep{steinke2022composition})
If a randomized algorithm $M$ satisfies $\rho-$zCDP, it satisfies ($\epsilon = \rho + 2\sqrt{\rho \log(1/\delta)},\delta)-$DP for all $\delta >0.$ Also, to obtain a given a target $(\epsilon, \delta)$-DP guarantee, it suffices to have $\rho-$zCDP with
\begin{equation*}
    \frac{\epsilon^{2}}{4\log(1/\delta)+4\epsilon} \leq \rho = \left( \sqrt{\log(1/\delta)+\epsilon} - \sqrt{\log(1/\delta}\right)^{2} \leq \frac{\epsilon^{2}}{4\log(1/\delta)}.
\end{equation*}
\end{lemma}

\subsection{Objective Perturbation for $(\epsilon,\delta)$-DP guarantee}

The current existing results that satisfy $(\epsilon,\delta)$-DP are not applicable to the multinomial model. Therefore, we extend Theorem 1 in \cite{kifer2012private} to to address this. Additional difference is that our result guarantees bounded DP, while Theorem 1 in \cite{kifer2012private} guarantees unbounded DP.

\begin{theorem}\label{thm: objective perturbation eps del}($(\epsilon,\delta)$-DP guarantee of objective perturbation). Let $\epsilon,\delta >0$ be given and $\ell(\theta;z)$ be convex and twice-differentiable with $\Vert \nabla \ell(\theta, z)\Vert_{2} \leq L$ and the eigenvalues of $\nabla^{2}\ell(\theta;z)$ is upper bounded by $\eta$ for all $\theta \in \Theta$ and $z \in Z$. In addition, let $R$ be the rank of $l(\theta;z)$. Then the objective perturbation satisfies $(\epsilon,\delta)$-DP when 
$\Delta \geq \frac{(1-q)R\eta}{\epsilon}$ and $\sigma \geq \frac{L\left(\sqrt{d+2\sqrt{d\log(2/\delta)}+2\log(2/\delta)}+\sqrt{d+2\sqrt{d\log(2/\delta)}+2\log(2/\delta)+2q\epsilon} \right)}{q\epsilon},$ for any $q \in (0,1)$.
\end{theorem}

\begin{proof}
    We modify the proof of Theorem 1 in \cite{kifer2012private} while leveraging the results in the proof of Theorem \ref{thm: mle privacy}. To begin with, fix $\epsilon >0$ and $\delta >0$ and take neighboring databases by $Z = (Z_{-n},z)^\top$ and $Z' = (Z_{-n},z')^\top$ without loss of generality. To have $(\epsilon,\delta)$-DP, we need to show
    \begin{equation*}
        e^{-\epsilon}\left(\text{pdf}(\hat{\theta}(Z')=a)-\delta \right) \leq \text{pdf}(\hat{\theta}(Z)=a) \leq e^{\epsilon}\left(\text{pdf}(\hat{\theta}(Z')=a\right) + \delta.
    \end{equation*}

Define $R(a) = \frac{\text{pdf}(\hat{\theta}(Z)=a)}{\text{pdf}(\hat{\theta}(Z')=a)}$ as in the proof of Theorem \ref{thm: objective perturbation} so we have
\begin{equation*}
        R(a) = \left\vert \frac{\text{det}(\nabla b(a;Z')}{\text{det}(\nabla b(a;Z)} \right\vert \frac{v(b(a;Z);\sigma)}{v(b(a;Z');\sigma)},
\end{equation*}
where $v$ is the pdf of a gaussian random variable.
Note then what we aim is to show $\log R(a)\leq \epsilon$ with probability $1-\delta$. 

\noindent
\textbf{Upper bound on log-likelihood ratio}

Recall from the proof of Theorem \ref{thm: mle privacy} that the ratio of determinants is bounded as:
\begin{equation*}
    \begin{split}
\left\vert \frac{\text{det}(\nabla b(a;Z'))}{\text{det}(\nabla b(a;Z))} \right\vert 
&\leq \left(1+\frac{\eta}{\Delta}\right)^{R} = \left(\frac{\Delta+\eta}{\Delta}\right)^{R}.
    \end{split}
\end{equation*}
Taking $\Delta = \frac{qR\eta}{\epsilon}$, we have $\left\vert \frac{\text{det}(\nabla b(a;Z'))}{\text{det}(\nabla b(a;Z))} \right\vert \leq \left(1+\frac{\epsilon\eta}{2 R\eta}\right)^{R} \leq e^{(1-q)\epsilon}.$

On the other hand, recall again that  $\frac{v(b(a;Z);\sigma)}{v(b(a;Z');\sigma)}$ is upper bounded by $\exp\left(\frac{1}{2\sigma^{2}}(4L^{2} + 4L\Vert b(a;Z) \Vert)\right)$. 
Note that $\Vert b(a;Z) \Vert$ is equal in distribution to $\sigma X$, for is a $\chi$-distributed random variable $X$ with degree of freedom $d$. 

To bound this term, note that for a $\chi^{2}$-distributed random variable $X$ with degree of freedom $d$ satisfies the following bound:

\begin{lemma}(Lemma 1 in \cite{laurent2000adaptive})\label{thm: chi concentration inequality}
    For $X^{2} \sim \chi^{2}(d)$, we have
    \begin{equation*}
        \mathbb{P}(X^{2}-d \geq 2\sqrt{dx}+2x) \leq e^{-x}.
    \end{equation*}
\end{lemma}

Combining Lemma \ref{thm: chi concentration inequality} with an observation that $\{X \geq x\} = \{X^{2} \geq x^{2}\}$ for any $x>0$, provided that $X$ is a nonnegative random variable, we have

\begin{equation*}
    \mathbb{P}(X \geq \sqrt{d + 2\sqrt{dx}+2x}) \leq e^{-x}.
\end{equation*}

\noindent
\textbf{Probability that the bound holds}
So far we have 
\[
\mathbb{P}\left(\frac{\Vert b(a;Z) \Vert}{\sigma} \geq \sqrt{d + 2 \sqrt{dx}+2x}\right) \leq e^{-x},
\]
where $x >0$. 

Now, let \texttt{GOOD} be the set $\{a \in \mathbb{R}^{d}: \frac{\Vert b(a;Z) \Vert}{\sigma} \geq \sqrt{d + 2 \sqrt{dx}+2x} \}$. We want the noise vector $b$ to be in the set \texttt{GOOD} with probability at leat $1-\delta$. Setting $x = \log(2/\delta)$ implies that $\delta = 2e^{-x}$. Replacing $x = \log(2/\delta)$, we have $\frac{v(b(a;Z);\sigma)}{v(b(a;Z');\sigma)} \leq \exp\left(\frac{1}{2\sigma^{2}}(4L^{2} + 4L\sigma\sqrt{d+2\sqrt{d\log(2/\delta)}+2\log(2/\delta})\right)$. Plugging $\sigma \geq \frac{L\left(\sqrt{d+2\sqrt{d\log(2/\delta)}+2\log(2/\delta)}+\sqrt{d+2\sqrt{d\log(2/\delta)}+2\log(2/\delta)+2q\epsilon} \right)}{q\epsilon},$ for any $q \in (0,1)$, we further simplify the upper bound as $e^{q\epsilon}$.

To complete the argument, note that we have:
\begin{equation*}
    \begin{split}
        \text{pdf}(\hat{\theta}(Z)=a) &= \mathbb{P}(b \in \texttt{GOOD}) \text{pdf}(\hat{\theta}(Z)=a \vert b \in \texttt{GOOD}) +\mathbb{P}(b \in \texttt{GOOD}^\complement) \text{pdf}(\hat{\theta}(Z)=a \vert b \in \texttt{GOOD}^\complement) \\
        &\leq \mathbb{P}(b \in \texttt{GOOD}) \text{pdf}(\hat{\theta}(Z)=a \vert b \in \texttt{GOOD}) + \delta \\
        &\leq e^{\epsilon}\mathbb{P}(b \in \texttt{GOOD}) \text{pdf}(\hat{\theta}(Z')=a \vert b \in \texttt{GOOD}) + \delta \\
        &\leq e^{\epsilon}\text{pdf}(\hat{\theta}(Z')=a) + \delta.
    \end{split}
\end{equation*}    
\end{proof}

To apply Theorem \ref{thm: objective perturbation eps del} in our multinomial model, we take $L=2, R=\min\{d,K-1\}$ and $\eta = 4$ as identified previously. Moreover, we take $q=1/2$. 

Note that given the total privacy budget for perturbed MLEs is $(\epsilon_{1},\delta_{1})$, each call of \texttt{PrivateMLE} should be implemented with privacy parameters $\left(\epsilon_{1}^{'},\delta_{1}^{'}\right)=\left(\frac{\epsilon_{1}}{\sqrt{8D_{\text{MLE}}\log(1/\delta_{1})}},\frac{\delta_{1}}{2D_{\text{MLE}}}\right),$ so that the at most $D_{\text{MLE}}$ compositions of \texttt{PrivateMLE} satisfies $(\epsilon_{1},\delta_{1})$-DP by Lemma \ref{thm: advanced composition}.

The following is a pseudocode to implement \texttt{PrivateMLE} that satisfies $(\epsilon_{1}^{'},\delta_{1}^{'})$-DP: 

\begin{algorithm}[htb!]
\caption{\texttt{PrivateMLE} for $(\epsilon_{1}^{'},\delta_{1}^{'})$-DP}\label{alg: private mle epsdel}
\begin{algorithmic}[1]
    \STATE \textbf{Input:} current time step $t$, dimension of context $d$, assortment size $K$, rank of the hessian matrix $R = \min\{d,K-1\}$, privacy parameters $\epsilon_{1}^{'}, \delta_{1}^{'}$
    \STATE Set $\Delta = \frac{2R}{\epsilon_{1}^{'}}, \sigma = \frac{4L\left(\sqrt{d+2\sqrt{d\log(2/\delta_{1}^{'})}+2\log(2/\delta_{1}^{'})}+\sqrt{d+2\sqrt{d\log(2/\delta_{1}^{'})}+2\log(2/\delta_{1}^{'})+2q\epsilon_{1}^{'}} \right)}{\epsilon_{1}^{'}}$
    \STATE Sample $b \sim N(0,\sigma^{2}I)$
    \STATE \textbf{Output: } $\hat{\theta}_{t} =
    \sum_{n=1}^{t}\sum_{i \in S_{n}\cup\{0\}}-y_{ti}\log p_{n}(i \vert S_{n},\theta)+ \frac{\Delta}{2}\Vert \theta \Vert_{2}^{2} + b^\top \theta$
\end{algorithmic}
\end{algorithm}

\subsection{Tree-based Aggregation for $(\epsilon,\delta)$-DP}
To instantiate the tree-based aggregation for gram matrix to satisfy $(\epsilon,\delta)$-DP, we follow the step introduced in \cite{shariff2018differentially}. According to \cite{shariff2018differentially}, to make the whole tree satisfy $(\epsilon_{2},\delta_{2})$-DP, it is sufficient to make each node in the tree algorithm needs to satisfy $(\epsilon_{2}/\sqrt{8m\log(2/\delta_{2})},\delta_{2}/2$)-DP. Applying the method introduced in \cite{shariff2018differentially} to our multinomial setting, it is sufficient to take such Gaussian noise as follows: sample $N' \in \mathbb{R}^{d \times d}$ with each $N'_{ij} \stackrel{iid}{\sim} N(0, \sigma_{cov}^{2})$, where $\sigma_{cov}^{2} = 32mK\log(4/\delta_{2})^{2}/\epsilon_{2}^{2}$. Then symmetrize to get $N = (N' + N'^\top)/\sqrt{2}$.

\begin{algorithm}[htb!]
\caption{\texttt{PrivateCov}}\label{alg: privatecov epsdel}
\begin{algorithmic}[1]
\STATE \textbf{Input: } privacy parameters $\epsilon_{2},\delta_{2}$, $T$
\STATE Set $\sigma_{cov}^{2} = 32mK\log(4/\delta)^{2}/\epsilon^{2}$, $m = 1+\lfloor \log_{2}T \rfloor$ 
\STATE \textbf{Initialization: } $p(l) = \hat{p}(l) = 0$ for all $l = 0,\cdots,m-1$
\FOR{$t \in [T]$}
\STATE Express $t$ in binary form: $t = \sum_{l=0}^{m-1} \text{Bin}_{t}(l)2^{l}$, where $\text{Bin}_{t}(l) \in \{0,1\}$
\STATE Find index of first one $l_{n} = \min \{l: \text{Bin}_{t}(l)=1\}$
\STATE Update $p$-sums: $p(l_{n}) = \sum_{l < l_{n}}p(l) + \sum_{i \in S_{t}}x_{ti}x_{ti}^\top$ and $p(l) = \hat{p}(l) = 0$ for all $l < l_{n}$
\STATE Inject noise: $\hat{p}(l_{n}) = p(l_{n}) + N \text{ where }N = (N' + N'^\top)/\sqrt{2} \text{ and }N'_{ij} = N'_{ji} \stackrel{\text{i.i.d.}}{\sim} \mathcal{N}(0,\sigma^{2}) $
\STATE Release $V_{t} = \sum_{l=0}^{m-1} \text{Bin}_{t}(l)\hat{p}(l)$
\ENDFOR
\end{algorithmic}    
\end{algorithm}

Finally, to run the main policy \texttt{DPMNL} with $(\epsilon_{1}+\epsilon_{2},\delta_{1}+\delta_{2})$-JDP guarantee, the confidence width should also be modified. The modified confidence width $\Tilde{\alpha}_{t}$ can be derived by similar argument to Equation (\ref{equ: mle estimation error}), which results in
\[
\Tilde{\alpha}_t = \sqrt{\left(\frac{d}{2}\right) \log\left(1 + \frac{t + 1}{d}\right) + \log(t + 1)} + \frac{4R}{\epsilon_{mle}\sqrt{K}} + \frac{\sqrt{4d\log(T)\sigma_{mle}^2}}{\sqrt{K}} + \sqrt{3\Tilde{\lambda}},
\]
where $\sigma_{mle}$ is given by plugging $\epsilon_{mle}$ into $\sigma$ in Theorem \ref{thm: objective perturbation eps del} and
\[
\Tilde{\lambda} = \sigma_{cov} \sqrt{m} \left(2 \sqrt{d} + 2 d^{1/6}\log^{1/3}d + \frac{6(1+(\log d/d)^{1/3})\sqrt{\log d}}{\sqrt{\log \left(1+\left(\frac{\log d}{d}\right)^{1/3}\right)}} 
+ 2 \sqrt{4\log T} \right).
\]
Thus, we implement the main policy \texttt{DPMNL} using the modified confidence width \( \tilde{\alpha}_{t} \) and \( \tilde{\lambda} \). The private MLE is obtained via Algorithm \ref{alg: private mle epsdel}, and the private covariance is computed using Algorithm \ref{alg: privatecov epsdel}. Since the assortment at any time \( t \) in \texttt{DPMNL} is determined by a function of the true context vector of user \( t \) and two privacy subroutines, whose composition satisfies \((\epsilon_{1} + \epsilon_{2}, \delta_{1} + \delta_{2})\)-DP according to Lemma \ref{thm: basic composition}, it follows from Lemma \ref{thm: billboard lemma epsdel} that the policy \texttt{DPMNL} satisfies \((\epsilon_{1} + \epsilon_{2}, \delta_{1} + \delta_{2})\)-JDP.

\subsection{Details on Processing Expedia Dataset}\label{sec: processing expedia}
To process the Expedia dataset, we adopt the approach described in \cite{lee2024low}. Specifically, we focus on data pertaining to the top destination with the highest search volume, recognizing that search queries inherently limit assortments to those within the specified destination.

In our data preprocessing, we exclude any columns that are missing in more than 90\% of searches. For the remaining data, particularly where there are missing values for users' historical average star ratings and prices, we apply a simple regression tree imputation to fill in missing values. For hotel features, we compute the mean values of features across different searches.

Additionally, to enhance the effectiveness of our model fitting, we discretize numerical features such as star ratings and review scores by converting them into categorical variables.

\section{Technical Lemma}\label{sec: technical lemma}

\begin{lemma}\label{thm: chen_mle_conv}{(Lemma A in \cite{chen1999strong})}
Let $H$ be a smooth injection from $\mathbb{R}^{d}$ to $\mathbb{R}^{d}$ with $H(x_{0}) = y_{0}.$ Define $\mathbf{B}_{\delta}(x_{0}) = \left\{ x \in \mathbb{R}^{d} : \Vert x - x_{0} \Vert \leq \delta \right\}$ and $\mathbf{S}_{\delta}(x_{0}) = \partial \mathbf{B}_{\delta}(x_{0}) = \left\{ x \in \mathbb{R}^{d} : \Vert x - x_{0} \Vert = \delta \right\}$. Then $\inf_{x \in \mathbf{S}_{\delta}(x_{0})} \Vert H(x) - y_{0} \Vert \geq r$ implies that
\begin{enumerate}
    \item $\mathbf{B}_{r}(y) = \left\{ y \in \mathbb{R}^{d} : \Vert y-y_{0} \Vert \leq r \right\} \subseteq H(\mathbf{B}_{\delta}(x_{0}))$.
    \item $H^{-1}(\mathbf{B}_{r}(y)) \subseteq \mathbf{B}_{\delta}(x_{0})$.
\end{enumerate}    
\end{lemma}

\begin{lemma}\label{thm: oh_lemma 8}{(Lemma8 in \cite{oh2021multinomial})} Suppose $x_{ti} \in \mathbb{R}^{d}$ satisfies $\Vert x_{ni} \Vert \leq 1$ for all $i$ and $n$. Let $\Sigma_{t} = \sum_{n=1}^{t}\sum_{i \in S_{n}}x_{ti}x_{ti}^\top$, where the cardinality of an assortment $S_{n}$ is $K$. Then $\det(\Sigma_{t})$ is increasing with respect to $t$ and we have
\begin{equation*}
    \det(\Sigma_{t}) \leq \left(\frac{tK}{d}\right)^{d}
\end{equation*}
\end{lemma}

\begin{lemma}\label{thm: zcdp_composition}(Theorem 13 in \citep{steinke2022composition})
Let $M_{1},M_{2}, \cdots M_{k}: \mathcal{X}^{n} \rightarrow \mathcal{Y}$ be randomized algorithms. Suppose $M_{j}$ is $\rho_{j}-$zCDP for each $j \in [k].$ Define $M: \mathcal{X}^{n} \rightarrow \mathcal{Y}^{k}$ by $M(x) = \left(M_{1}(x), M_{2}(x), \cdots M_{k}(x) \right),$ where each algorithm is run independently. Then $M$ is $\rho-$zCDP for $\rho=\sum_{j=1}^{k}\rho_{j}.$
\end{lemma}

\begin{lemma}\label{thm: basic composition}(Basic Composition, Theorem 3.16. in \cite{dwork2014algorithmic})
Let $M_{i}$ be an $(\epsilon_{i},\delta_{i})$-differentially private mechanism for $i \in [k]$. Then if composition $M_{[k]}$ is defined to be $M_{[k]} = (M_{1},\cdots,M_{k})$, then $M_{[k]}$ is $(\sum_{i=1}^{k}\epsilon_{i},\sum_{i=1}^{k}\delta_{i})$-differentially private.
\end{lemma}

\begin{lemma}\label{thm: advanced composition}(Advanced Composition, Corollary 3.21 in \cite{dwork2014algorithmic})
    Given target privacy parameters $0 < \epsilon' <1$ and $\delta'>0$, to ensure $(\epsilon',k\delta+\delta')$ cumulative privacy loss over $k$ mechanisms, it sufficies that each mechanism is $(\epsilon,\delta)$-differentially private, where
    \begin{equation*}
        \epsilon = \frac{\epsilon'}{2\sqrt{2k\log(1/\delta')}}.
    \end{equation*}
\end{lemma}

\begin{lemma}\label{thm: kifer l2 bound}(Lemma 28 in \cite{kifer2012private}) For $b \sim \mathcal{N}\left(0, \sigma^{2}I_{d}\right)$, we have with probability at least $1 - \gamma$,
\begin{equation*}
    \Vert b \Vert_{2} \leq \sqrt{2d\sigma^{2}\log\left(\frac{1}{\gamma}\right)}.
\end{equation*}
\end{lemma}

\begin{lemma}\label{thm: chi distn bound}
    Let $W$ be a chi-distributed random variable with degree of freedom $d$, or $W \sim \chi(d)$. Then for $\sigma >0$, we have
    \begin{equation*}
        \mathbb{E}(\sigma W) \leq \sigma\sqrt{d}.
    \end{equation*}
\end{lemma}

\begin{proof}
For \( W \sim \chi(d) \), we write \( W = \sqrt{X_1^2 + X_2^2 + \cdots + X_d^2} \), where \( X_i \stackrel{iid}{\sim} N(0, 1) \). 
Since \( \sqrt{x} \) is a concave function, we can apply Jensen’s inequality:
\[
\mathbb{E}(W) = \mathbb{E}\left(\sqrt{X_1^2 + X_2^2 + \cdots + X_d^2}\right) \leq \sqrt{\mathbb{E}(X_1^2 + X_2^2 + \cdots + X_d^2)},
\]
where $\mathbb{E}(X_1^2 + X_2^2 + \cdots + X_d^2) = d$ since $\mathbb{E}(X_i^2) = 1 \text{ for each } i.$
Combining these results gives:
\[
\mathbb{E}(W) \leq \sqrt{d}.
\]
Now, scaling by \( \sigma \), we have 
\[
\mathbb{E}(\sigma W) = \sigma \mathbb{E}(W) \leq \sigma \sqrt{d}.
\]
\end{proof}

\end{document}